\documentclass{aux/siamonline190516}
\usepackage[utf8]{inputenc} 
\usepackage[T1]{fontenc}    
\usepackage{aux/macros}
\usepackage{graphics,caption,subcaption}
\usepackage{float}
\usepackage{wrapfig}
\newcommand\blfootnote[1]{%
  \begingroup
  \renewcommand\thefootnote{}\footnote{#1}%
  \addtocounter{footnote}{-1}%
  \endgroup
}

\usepackage{nameref}

\def\eps{\varepsilon}

\usepackage{amsmath}           
\usepackage{etoolbox}          

\usepackage{enumitem}
\setlist[enumerate,1]{leftmargin=1.5em}

\title{On the Inconsistency of Kernel Ridgeless Regression in Fixed Dimensions}

\author{Daniel Beaglehole$^*$ 
Mikhail Belkin$^{\dagger,^*}$
Parthe Pandit$^\dagger$}

\def\X{X}

\begin{document}
\blfootnote{$^*$Computer Science and Engineering, University of California, San Diego}
\blfootnote{$^\dagger$\Halicioglu~Data Science Institute,  University of California, San Diego}
\maketitle
\begin{abstract}
``\textit{Benign overfitting}'', the ability of certain algorithms to interpolate noisy training data and yet perform well out-of-sample, has been a topic of considerable recent interest. We show, using a fixed design setup, that an important class of predictors, kernel machines with translation-invariant kernels, does not exhibit benign overfitting in fixed dimensions. In particular, the estimated predictor does not converge to the ground truth with increasing sample size, for any non-zero regression function and any (even adaptive) bandwidth selection. To prove these results, we give exact expressions for the generalization error, {and its decomposition in terms of an approximation error and an estimation error that elicits a trade-off based on the selection of the kernel bandwidth.} 
Our results apply to commonly used translation-invariant kernels such as Gaussian, Laplace, and Cauchy.
\end{abstract}

                            \section{Introduction}
Recent empirical evidence has shown that certain algorithms, contrary to classical learning theory, can interpolate noisy data, \ie, achieve zero training error, while still generalizing well out-of-sample, that is, exhibiting low test error \cite{zhang2021understanding, belkin2019reconciling,neyshabur2014search}.
This phenomenon of ``benign overfitting'' (using the terminology of~\cite{bartlett2020benign}) has been rigorously analyzed  for certain parametric methods such as linear regression, and random feature regression \cite{bartlett2020benign, hastie2022surprises,muthukumar2020harmless,belkin2020two}, as well as non-parametric methods such as  kernel regression with singular kernels~\cite{devroye1998hilbert,belkin2018overfitting,belkin2019does}. 

Many theoretical results in this direction assume a high-dimensional regime where the data dimension $d$ grows with the sample size $n$. However, it remains unclear whether this phenomenon is common when the data dimension is fixed. In particular, it has been an open question whether popular practical algorithms, such as kernel machines~\cite{scholkopf2018learning, gyorfi2002distribution},  exhibit benign overfitting. 

Indeed, the work of~\cite{liang2020just} showed that interpolating kernel machines, also known as kernel ridgeless regression, can be consistent in high dimension, i.e., can converge to an optimal predictor given enough data. On the other hand, the work of Rakhlin and Zhai~\cite{rakhlin2019consistency} showed that for the specific case of Laplace kernel, kernel ridgeless regression is inconsistent in fixed dimensions even with a data-adaptive bandwidth. This is significant  as the kernel bandwidth hyperparameter can have a large effect on the estimated predictor, and indeed can be set adaptively in high dimensions to achieve consistency. 

In this work, we show that this lack of benign overfitting in fixed dimension is in fact a general property of a broad class of kernel machines. Specifically, we prove that consistency does not hold for the  widely used class of translation-invariant kernels, \ie, kernels that depend only on the difference of the two inputs, under mild spectral conditions. Important examples of such kernels include the Gaussian, Laplace, and Cauchy kernels.

Our counterexample uses a simple data model of the grid on the unit circle for $d=1$, and, in higher dimensions, a multidimensional torus, i.e., the product of unit circles, when $d>1$. For clarity, we outline the $d=1$ case in the main body of the paper, and generalize to $d>1$ in \Cref{app:highd}. 

To prove these results, we derive exact expressions for the generalization mean-squared-error in terms of the Fourier series of the chosen kernel. These exact expressions elucidate the trade-off between approximation and estimation errors when choosing the bandwidth parameter. Our key insight is that while a small bandwidth reduces the estimation error, it worsens the approximation error. Our exact expressions enable us to provide a constant lower bound on the generalization error as the number of samples grows to infinity.

\paragraph{Related work} 
Several recent works have demonstrated the existence of benign interpolation in high dimensions (e.g., when dimension is scales linearly with the number of samples). In this setting, the generalization bounds for linear and random feature interpolation 
depend on the rate of decay of eigenvalues 
\cite{bartlett2020benign, hastie2022surprises}. 
For example, \cite{Mei_hypercontractivity} derives asymptotic risk curves in high dimensions for linear ridge regression and featurized linear ridge regression. Similarly \cite{Mei_RF_Asymptotics} describes the asymptotic behavior of random feature regression, deriving double descent curves. As these works showcase, interpolation is benign in these high dimensional settings, typically proportional asymptotics. Another work considers the consistency of rotation-invariant kernels in high-dimensions \cite{donhauser2021rotational}. 

In contrast, we consider the case of fixed dimensions. We show that in fixed dimensions, interpolation with kernel machines is inconsistent. We also strengthen our result by showing that this conclusion holds regardless of an adaptive bandwidth selection, which is often necessary to achieve consistency for high dimensional settings, e.g. \cite{devroye1998hilbert,belkin2018overfitting,belkin2019does,liang2020just}.  

\section{Problem setup}

\paragraph{Notation}  We denote functions by lowercase letters $a$, sequences by uppercase letters $A$,
vectors by lowercase bold letters $\a$, matrices by uppercase bold letters $\A$. Sequences are indexed using square-brackets, $A[k]$ where $k\in\Integer.$ For vectors, functions, sequences, $\inner{\a,\b},\inner{a,b},\inner{A,B}$ denote their Euclidean, $L^2$, and $\ell^2(\Integer)$ inner products respectively,
while $\norm{\a},\norm{a},\norm{A}$ denote corresponding induced norms, and $\norm{\a}_1,\norm{a}_1,\norm{A}_1$ denote their respective $1$-norms. Like the $L^1$ norm, other norms or inner products will be pointed out explicitly. For a nonnegative integer $N$, we denote the set $\set{0,1,\ldots,N-1}$ by $[N].$ We use $j$ to denote $\sqrt{-1}$, and an overline, $\wb{\a}$, to denote elementwise complex conjugation. The asymptotic \textit{big-Oh} notation $O_n(\cdot),\Omega_n(\cdot),o_n(\cdot),\omega_n(\cdot),$ have their usual meaning where the limit is with respect to $n$. 

We use $N\in\Natural$ as a \textit{resolution} hyperparameter (explicitly defined in \Cref{eqn:resolution}). For a sequence $G\in\ell^1(\Integer)$, and a fixed $N,$ we define an \textit{$N$-hop subsequence} $G_{\ell} \in \ell^1(\Integer)$ as
\begin{align}\label{eq:hop}
G_\ell = \set{G[mN + \ell]}_{m\in\Integer}\quad \text{ defined for }\ell\in\curly{0,1,\ldots, N-1}.
\end{align}

\paragraph{Nonparametric regression} We consider a supervised learning problem in the fixed design setting where we have $n$ labeled samples $(x_i, y_i)\in\mc X\times \mc Y \subseteq\Real^d\times\Real$, with labels generated as,
\begin{align*}
    y_i = f^*(x_i) + \xi_i,
    \qquad\qquad\xi_i \overset{\rm i.i.d.}\sim \Prob_\xi, \qquad\qquad \forall\, i \in [n]~,
\end{align*}
for some unknown target function $f^*.$ The noise distribution $\Prob_\xi$ is centered with a finite variance $\sigma^2>0$. We assume this distribution is independent of the chosen data $\set{x_i}$ and target $f^*$.

For a sequence of datapoints $X_n$, the estimation task is to propose an estimator $\wh f_n=\wh f_n(X_n,\y):\mc X\rightarrow\Real,$ where $\y=(y_i)\in\Real^n$ is the vector of all labels on these data. An estimator's performance (or generalization error) is measured in terms of its mean squared error,
\begin{align*}
    \MSE{\wh f_n, f^*}:=\norm{\wh{f}_n-f^*}^2 =\int_{\mc X} \rbrac{\wh{f}_n(x)-f^*(x)}^2\dif x.
\end{align*}

\paragraph{Weak consistency \cite{gyorfi2002distribution}} 
For a target function $f^*$, a sequence of estimators $\set{\wh f_n}$ is said to be weakly consistent if,
\begin{align*}
\lim_{n\rightarrow\infty}\Exp_{\xivec}\MSE{\wh f_n, f^*} = 0.
\end{align*}
In this paper we show that a certain sequence -- kernel ridgeless regression estimators -- is weakly inconsistent, \ie, $\lim_{n\rightarrow\infty}\Exp_{\xivec}\MSE{\wh f_n, f^*} > 0.$ Note that weak inconsistency implies inconsistency in the strong sense as well.

\paragraph{Kernel interpolation} (also known as kernel ridgeless regression) For an RKHS $\mc H$, the kernel interpolation estimator is given by,
\begin{align}\label{eq:kernel_interpolation}
    \MoveEqLeft\wh f_n = \argmin{f\in\Hilbert}\norm{f}_\Hilbert\qquad\subjectto{\ \ f(x_i)=y_i} \quad \text{  for  } i = \{1,2,\ldots, n\}.
    \end{align}
    The name ridgeless is due to the fact that the solution is equivalent to the following \textit{kernel ridge regression} problem in the limit
    \begin{align}
        \MoveEqLeft\wh f_n = \lim_{\lambda\rightarrow 0^+} 
        \underbrace{\left(\argmin{f\in\Hilbert} \sum_{i=1}^n (f(x_i)-y_i)^2 + \lambda\norm{f}^2_\Hilbert\right)}_{:= \wh f_{n,\lambda}}.
    \end{align}
    
    Every RKHS is in one-to-one correspondence with a positive definite kernel function $K:\mc X\times\mc X\rightarrow\Real.$ Define the kernel matrix $\Kmat=(k(x_i,x_j))$ of pairwise evaluations of the kernel on the training data. 
    Due to the representer theorem \cite{scholkopf2001generalized},  the solution to \cref{eq:kernel_interpolation} lies in the span of $n$ basis functions $K(x_i,x)$ and can be written as
    \begin{align}
    \MoveEqLeft\wh{f}_n(x) = \sum_{i=1}^n\wh\alpha_i K(x_i,x),\qquad\wh\alphavec=(\wh\alpha_i)\in\Real^n\qquad\wh\alphavec:=\Kmat\inv\y,\tag{Kernel interpolation}
\end{align}
where $\y\in\Real^n$ is the vector of all labels. The above follows as a direct consequence of $\wh f_{n,\lambda}=K(\cdot,X)(\Kmat+\lambda I_n)\inv\y$, and that $\wh f_n=\lim_{\lambda\rightarrow0^+}f_{n,\lambda}$. The matrix $\Kmat$ is invertible because the kernel is positive definite, otherwise interpolation in an RKHS is not always possible. The (Riesz) representer of a given kernel $K$ at a datum $x_\star$ is an element of $\Hilbert$, denoted by $K(x_\star,\cdot): \mc X \rightarrow \mathbb{R}$. It is the evaluation functional of $x_\star\in\mc X$, \ie, $\inner{f,K(x_\star,\cdot)}_\Hilbert=f(x_\star)$ for all $f\in\Hilbert.$ The basis functions above are thus the representers of the training data $\set{x_1,x_2,\ldots x_n}.$

We define the \textit{restriction operator} $R_n$, and its adjoint, the  \textit{extension operator} $R_n^*$, as follows:
\begin{align}\label{eq:restriction_extension}
    &R_n:\Hilbert\rightarrow \Real^n 
    &R_n f &:= (f(x_i)) \in \Real^n,\qquad \forall\,f\in \Hilbert\\
    &R_n^*:\Real^n\rightarrow\Hilbert &R_n^*\alphavec &:= \sum_{i=1}^n\alpha_i K(x_i,\cdot) \in \Hilbert,\qquad\forall\,\alphavec=(\alpha_i)\in\Real^n
\end{align}
that evaluates the function on the data. Here, since $L^2_n{\cong}\,\Real^n$ are isometric, we are abusing notation in favour of simpler expressions.
This gives us the following equations
\begin{align*}
    \y = R_n f^* + \xivec,\qquad\text{and}\qquad
    \wh f_n = R_n^*\Kmat\inv\y.
\end{align*}

For an RKHS we have two data dependent operators, the \textit{integral operator} and the \textit{empirical operator}, respectively given by,
\begin{align}
    &\mc T_{K}{f}(x) = \int_{\mc X} K(x,z)f(z)\dif z,\\
    &\mc T_{K}^n{f}(x) = \sum_{z \in X_n} K(x,z)f(z).\label{eq:def:emp_cov}
\end{align}

Eigenfunctions of $\mc T_K$ that form a countable orthonormal basis of $L^2(\mc X)$ can be used to provide an alternate representation for the $\Hilbert$-norm via the identity,
\begin{align*}
    \inner{f,g}_\Hilbert = \sum_{k\in\Integer}\frac{\inner{f,\varphi_k}\inner{g,\varphi_k}}{\sigma_k} \qquad \norm{f}_\Hilbert^2 = \sum_{k\in\Integer}\frac{\inner{f,\varphi_k}^2}{\sigma_k}
\end{align*}
where $(\sigma_k,\varphi_k)$ is an eigen-pair, \ie, $\mc T_{K}\varphi_k=\sigma_k\cdot\varphi_k$, with $\sigma_k\in\Real_+$ and $\varphi_k\in L^2$.

\paragraph{Fourier analysis:} We recall some useful quantities from Fourier analysis to be used later.

\begin{definition}[Fourier basis]
Let $\phi_k(x) = e^{jkx}$ for $k\in\Integer$, which satisfy
\begin{align*}
\inner{\phi_k,\phi_\ell}:=\int_{-\pi}^\pi\phi_k(t)\overline{\phi_\ell(t)}\,\frac{\dif t}{2\pi}=\frac1{2\pi}\int_{-\pi}^\pi e^{j(k-\ell)t}\dif t=\begin{cases}
    1 & k=\ell\\
    0 & k\neq \ell
\end{cases}.
\end{align*}
\end{definition}
The normalization factor $\frac1{2\pi}$ comes from the uniform density on $[-\pi,\pi).$
An important tool in our analysis is the Fourier series representation of functions $\mc X\mapsto\Real.$ In general, any integrable function $\Real\rightarrow\Real$ periodic with period $2\pi,$ admits such a representation.

\begin{definition}[Fourier Series]
For $f\in L^1_{\unitcircle}$, let $F$ be the Fourier series indexed by $k\in\Integer$,
\begin{align*}
    f(t) &= \sum_{k\in\Integer} F[k] \phi_k(t)=\sum_{k\in\Integer} F[k]e^{jkt},\qquad\forall\, t\in\unitcircle\\
    F[k] &= \inner{f,\phi_k}=\int f(t)\wb{\phi_k(t)}\,\frac{\dif t}{2\pi}=\frac1{2\pi}\int_{-\pi}^\pi f(t)e^{-jkt}\dif t,\qquad \forall\,k\in\Integer~.
\end{align*}
\end{definition}

\begin{definition}[DFT Matrix]
The normalized discrete Fourier transform (DFT) matrix is 
\begin{align*}
\U = \begin{bmatrix}
    \u_0   & \cdots & \u_{N-1}
\end{bmatrix},
\qquad 
\u_{\ell} = \frac{1}{\sqrt{N}}
\begin{bmatrix}
1 & e^{-j\frac{2\pi}N\ell} &\ldots & e^{-j\frac{2\pi}{N}(N-1)\ell}
\end{bmatrix}\tran,\quad \ell \in [N].
\end{align*}
\end{definition}
Notice that $\U\U\herm=\U\herm\U=\I$, where we use $\herm$ to denote the conjugate transpose (hermitian) of a matrix.

\begin{proposition}[Parseval's theorem]
\label{thm:parseval_continuous}
For a continuous function $f:\unitcircle\rightarrow\Real$ with Fourier series $F$, 
\begin{align*}
    \frac1{2\pi}\int_{-\pi}^\pi |f(t)|^2\dif t = \sum_{k\in\Integer} |F[k]|^2.
\end{align*}
\end{proposition}

\section{Model}
\label{sec:model}

We now describe our setting and state our main result: kernel interpolation is weakly inconsistent.

\paragraph{Data design (grid on the unit circle)}

We describe the case of $d=1$ and focus on $\mc X = {\unitcircle}$, viewed as the unit circle. An extension to $d>1$ is deferred to \Cref{app:highd} where we consider $\unitcircle^d$.
We consider discrete, evenly-spaced grids indexed by a \textit{resolution} hyperparameter $N\in\Natural,$ given by
\begin{align}
\label{eqn:resolution}
    \X_N = \set{x_0,\ldots,x_{N-1}} \qquad x_i := \frac{2\pi }{N}i-\pi \qquad \forall\,i=0,\ldots, N-1.
\end{align}

We call $N$ the resolution parameter of the grid on $\unitcircle$, and assume $N$ is even for simplicity. Observe that Riemannian sums over the grid $\X_N$ for integrable functions converge to integrals on the continuum $\unitcircle$. Alternatively, the empirical distribution on the grid weakly converges to the uniform measure on the continuum. Note for $d=1,$ the total number of samples $n$ equals the resolution $N.$

For $d>1,$ we consider $\mc X=\unitcircle^d$, the product of $d$ unit circles, and the respective grids, along each dimension. Thus $N$ is the number of samples per dimension, whereby the total number of samples $n=N^d$.

\paragraph{Translation-invariant kernels}

We consider (periodic) kernels parameterized by a positive bandwidth parameter\footnote{In machine learning literature, the bandwidth may often be denoted as $1/M$ instead.} $M$,
\begin{align*}
    K(x,x^\prime) = g\round{M(x-x^\prime \mod \unitcircle)},\qquad x,x^\prime\in\mc X
\end{align*}
for some even function $g: \mathbb{R} \rightarrow \mathbb{R}$,
where we denote,
\begin{align}\label{eq:def:mod_unitcircle}
    \theta \mod \unitcircle = ((\theta + \pi) \mod 2\pi) - \pi \in \unitcircle~.
\end{align}

We denote the RKHS corresponding to $K$ by $\Hilbert.$ For ease of notation, when $M=1,$ we refer to this as the base kernel and the base RKHS $\Hilbert_0$.
Define $G_0,G:\Integer\rightarrow\Complex$ as the Fourier series of $g$, \ie,
\begin{subequations}\label{eq:def:fourier_series_of_kernel}
\begin{align}\label{eq:FS_of_g}
    g(M(\theta \mod \unitcircle)) &= \sum_{k\in\Integer} G[k] \exp(jk\theta)~,  &G[k] &= \frac{1}{2\pi}\int_{-\pi}^\pi g(M\theta)\exp(-jk\theta)\dif\theta~,\\
    g(\theta\mod\unitcircle) &= \sum_{k\in\Integer} G_0[k] \exp(jk\theta)~,
    &G_0[k] &= \frac{1}{2\pi}\int_{-\pi}^\pi g(\theta)\exp(-jk\theta)\dif\theta~.
\end{align}
\end{subequations}
While usually the bandwidth scales the input, we note our analysis also holds for different mechanisms that satisfy the kernel assumptions given later. For symmetric positive definite kernels, $g$ is an even function whereby we have that $G$ is real. Furthermore $g$ is real whereby,
\begin{align*}
    G[-k] = G[k] > 0 \quad \forall\, k \in \Integer~.
\end{align*}

\begin{proposition}\label{prop:eigenvectors_of_Kmat} $\u_\ell$ and $\wb{\u_\ell}$ are eigenvectors of $\Kmat=(K(x_i,x_j))\in\Real^{N\times N},$ with eigenvalue $\lambda_\ell=N\norm{G_\ell}_1$, \ie, $\Kmat\u_\ell = \lambda_\ell\u_\ell$ and $\Kmat\wb{\u_\ell} = \lambda_\ell\wb{\u_\ell}$.
Furthermore,
\begin{align*}
    \Kmat = \sum_{\ell=0}^{N-1}\lambda_\ell\u_\ell\u_\ell\herm,\qquad\Kmat\inv = \sum_{\ell=0}^{N-1}\frac1{\lambda_\ell}\u_\ell\u_\ell\herm,\qquad\Kmat^{-2} = \sum_{\ell=0}^{N-1}\frac1{\lambda_\ell^2}\u_\ell\u_\ell\herm.
\end{align*}
\end{proposition}

\begin{proposition}\label{prop:eigenfunctions_of_TK}
For any $M>0,$ the Fourier basis are eigenfunctions of the kernel integral operator $\mc T_{K}$ with eigenvalues $G$, \ie, we have,
\begin{align*}
    \mc T_{K}\phi_k=G[k]\cdot\phi_k~.
\end{align*}
\end{proposition}
The proofs to these propositions are provided in \Cref{sec:technical:eigenfunctions}.

For $X_N$, we define the \textit{restriction operator} $R_N$, and its adjoint, the \textit{extension operator},
\begin{align}\label{eq:restriction_extension_H}
    &R_N:\Hilbert\rightarrow \Real^N~,\qquad R_N f = \round{f\round{\frac{2\pi}N(i-1)-\pi}}_i \in\Real^N~,\\
    &R_N^*:\Real^N\rightarrow \Hilbert~, \qquad R_N^*\alphavec := \sum_{i=0}^{N-1} \alpha_i K(x_i,\cdot)\in\Hilbert~.
\end{align}
We also use the notation
\begin{align*}
    \inner{\alphavec,K(\X_N, \cdot)}_N := \sum_{i=0}^{N-1} \alpha_i K(x_i,\cdot)
\end{align*}
to keep expressions simple. With this notation, the labels and the kernel interpolator can be written as
\begin{align}\label{eq:kernel_interpolation_as_span}
    \wh f_N = R_N^*\Kmat\inv\y= \inner{\Kmat\inv\y,K(\X_N,\cdot)}_N~.
\end{align}
\begin{definition}[Span of Riesz Representers]
Functions in the range of $R^*_N$, and of $\mc T_K^N$, are in the span of the representers $\{K(x_i,\cdot)\}_{i=1}^N$. \end{definition}

\paragraph{Target function} We assume the target function lies in the base RKHS $\Hilbert_0$, \ie, $\Hilbert$ with $M=1$, and has a norm $\norm{f^*}_{\Hilbert_0} = O_{M,N}(1)$.
As the target function is defined on the unit circle, it admits a Fourier series, 
\begin{align}\label{eq:target_fn}
f^* = \sum_{k \in \mathbb{Z}} V[k] \phi_{k}~.  
\end{align}
To keep derivations simple, we will assume, without loss of generality, that $V[k] \in \Real$ for all $k$ (i.e. the target function is even). It is straightforward to extend this argument to all $f^*$. We can decompose $f^*$ into an even and odd component (by $f^*(x) = \frac{f^*(x) + f^*(-x)}{2} + \frac{f^*(x) - f^*(-x)}{2}$). The even component will only have a cosine series (and hence real $V[k]$), and the odd component will only have a sine series (imaginary $V[k]$). The argument for the case of targets with imaginary $V[k]$ is identical to that for targets with real $V[k]$. Even and odd functions are in orthogonal subspaces of $L^2$ and of $\mc H$, whereby for general complex $V[k]$, the errors we derive are the sum of the errors of the even and odd components, and the arguments go through.

Recall the definition of the restriction and extension operators in \cref{eq:restriction_extension_H}.
Let $P_{\X}$ be the $L^2$-projection operator onto the span of the representers, \ie,
\begin{subequations}\label{eq:projection}
\begin{align}
    P_{\X}f &:= \argmin{h\in\Hilbert}\curly{\norm{f-h}\ \bigg|\ h=\sum_{i=1}^N\alpha_i K(x_i,\cdot) \text{ for some } (\alpha_i)\in\Real^N}~,\\
    \alphavec^*&:=(\alpha_i^*)\qquad \text{such that}\qquad P_{\X}f^* =\sum_{i=1}^N\alpha^*_i K(x_i,\cdot)~,\\
    f^*_\perp&:=f^*-P_{\X}f^*,
\end{align}
\end{subequations}
where $f^*_\perp$ is orthogonal to all functions in $\Span\curly{K(x_i,\cdot)}.$
An immediate identity using the evaluation operator $R_N$ is,
\begin{align*}
    R_N P_{\X}f^*=\Kmat\alphavec^*\qquad\text{and}\qquad R_N^*\alphavec^*=P_{\X}f^*~.
\end{align*}

We can decompose the target function as
\begin{align*}
    f^* &= P_{\X}{f^*} + f^*_{\perp} 
    = \sum_{i=0}^{N-1} \alpha^*_i K(x_i,\cdot) + f^*_{\perp}
    = \inner{\alphavec^*,K(\X_N,\cdot)}_N + f^*_{\perp}~.
\end{align*}
Using this, the vector of labels, and the kernel interpolation estimator can be written as,
\begin{align}
    \y&=R_Nf^* + \xivec = R_N P_{\X} f^* + R_N f_\perp^* + \xivec =  \Kmat\alphavec^* + R_N f_\perp^* + \xivec~,\\
    \wh{f}_N&=R_N^*\Kmat\inv\y= P_{\X}f^* + \inner{\Kmat\inv R_N f_\perp^*,K(\X_N,\cdot)}_N +
    \inner{\Kmat\inv \xivec,K(\X_N,\cdot)}_N,\label{eq:kernel_interpolator_decomposition}
\end{align}
where we have used the expression from \Cref{eq:kernel_interpolation_as_span}.
\section{Main result: Inconsistency of kernel interpolation}

Our main result holds under certain assumptions on the translation-invariant kernels. Below, we assume $M',i,i^*$ are all non-negative integers. Recall that $G$ is the Fourier series of the kernel function, see \Cref{eq:FS_of_g}. Note $G$ depends on $M$ but $G_0$, the Fourier series of the kernel corresponding to $\Hilbert_0$ - the base RKHS, does not.

\begin{assumption}[Integrability]
\label{ass:scale}
We assume the kernel is integrable. In particular, the integral $\displaystyle{\int_{-\pi}^\pi g(Mx) \dif x}$ exists and is finite for all $0<M< \infty$.
\end{assumption}

\begin{assumption}[Spectral Tail]
\label{ass:tail}
For all $k \in \Integer_{\geq 0}$, there exists a constant $C_1 > 0$ such that,
\begin{align}\label{eq:tail_assumption}
    |G[M^\prime k + i]| \leq \frac{C_1 |G[i]|}{1 + k^2}
\end{align}
holds for all $M' \geq M > 0$ and for all $i \leq M'$, except $o_{M'}(M')$ many.
\end{assumption}

\begin{assumption}[Spectral Head]
\label{ass:head}
There exist constants $C_2, C_3\in\Real_+$ and $i^*\in\Integer_{\geq 0}$ such that for $M\geq C_2$, we have that for all $0 \leq M' < M$, $|G[i^*]| \leq C_3  |G[i^*+M']|$  and $\abs{G_0[i^*]}>0$.
\end{assumption}

To simplify analysis for many kernel functions, we give a sufficient condition that implies Assumptions \ref{ass:scale}-\ref{ass:head}, and is easy to verify for many functions.

\begin{condition}[Monotonic Boundedness]
\label{con:monotonic}
There exist constants (independent of the bandwidth $M$) $c,C,C',C''>0$, a constant $c'(M)>0$ (that may depend on $M$), and a bounded, monotonically decreasing function $f:\Real_{\geq 0}\rightarrow\Real$  with (i) $0 < \frac{f(x+k)}{f(x)} \leq \frac{C''}{1+k^2}$ for all $x \in \Real_{\geq 0}$,$k \in \Integer$, and (ii) $\frac{f(1 + x)}{f(x)} \geq C'$ for $0 \leq x \leq 1$, such that 
\begin{align*}
c f(\tfrac{k}M) \leq \frac{G[k]}{c'(M)} \leq C f(\tfrac{k}M)
\end{align*} for all $k \in \Integer_{\geq 0}$, $M \in \Real^+$.
\end{condition}

The proof of the following propositions are provided in \Cref{sec:assumptions}.

\begin{proposition}
\label{prop:monotone_assumptions}
If the Fourier series coefficients $G[i]$ satisfy \Cref{con:monotonic} (Monotonic Boundedness), then the kernel satisfies Assumptions \ref{ass:scale}-\ref{ass:head}.
\end{proposition}
\begin{proposition}
\label{lemma:ex_kernels_inconsistent}
The Gaussian $G[k] = \exp(-\frac{k^2}{M^2})$, Laplacian $G[k] = \frac{1}{1 + \frac{k^2}{M^2}}$, and Cauchy $G[k] = \exp(-\frac{|k|}{M})$ kernels (wrapped on the circle) satisfy \Cref{con:monotonic} (Monotonic Boundedness).
\end{proposition}

We comment on each of the Assumptions \ref{ass:scale}-\ref{ass:head} below.

\begin{remark}[Note on \Cref{ass:scale}]
A sufficient condition for our result is $|G[k]| < \infty$ for all $k \in \Integer$. Assumption {\ref{ass:scale}} implies this inequality by the definition of the Fourier coefficients.
\end{remark}

\begin{remark}[Square-integrable derivative $\implies$  Assumption \ref{ass:tail}]
Combined with Assumption {\ref{ass:scale}}, the exchange formula \cite{lieb2001analysis} for Sobolev spaces implies Assumption {\ref{ass:tail}} is equivalent to the kernel and its first derivative being $L^2$-integrable (for fixed bandwidth). Therefore, this assumption can be viewed as a condition on the smoothness of the kernel.
\end{remark}

\begin{remark}[Interpretation of \Cref{ass:head}]
Intuitively, Assumption {\ref{ass:head}} enforces flatness in the frequency domain, or equivalently, sharpness in the $\mc X$-domain. A larger bandwidth $M$ leads to a longer sequence of similar coefficients $G[k]$ for $k \in \{i^*,\ldots,i^*+M\}$, giving a sharper kernel in the $\mc X$-domain.
\end{remark}

We present the main results in the following theorems. Recall that $\Hilbert_0$ is the base RKHS.

\begin{theorem}[Inconsistency for all functions when $G$ is monotonically bounded]
\label{thm:inconsistency_monotonic}
Consider a fixed non-zero regression function $f^*$ that (i) has square-integrable zeroth and first derivatives, and (ii) can be expressed as a convergent Fourier series. Then, interpolation with 
a real-valued translation-invariant kernel satisfying \Cref{con:monotonic} 
is inconsistent for $f^*$, for any bandwidth, even if chosen adaptively.
\end{theorem}

Recall the definition of the base RKHS $\Hilbert_0$, above equation \eqref{eq:def:fourier_series_of_kernel}, corresponding to the kernel with bandwidth $M=1$.

\begin{theorem}[Inconsistency for all Bandwidths]
\label{theorem:inconsistent_all_band}
For any translation-invariant kernel satisfying Assumptions \ref{ass:scale}-\ref{ass:head}, there exists a function with constant $\mathcal{H}_0$-norm for which kernel interpolation is inconsistent for any bandwidth, even adaptive to the data set.
\end{theorem}

\begin{theorem}[Inconsistency for all Functions ($M < N$)]
\label{theorem:inconsistent_all_f}
For any translation-invariant kernel satisfying Assumptions \ref{ass:scale}-\ref{ass:tail}, with any (even data-adaptive) bandwidth $M \leq N$, kernel interpolation is inconsistent for all targets that can be expressed as convergent Fourier series. In particular, kernel interpolation with a fixed bandwidth is inconsistent for all such targets.
\end{theorem}

To prove these results we apply Fourier analysis to compute an exact expression for the MSE of kernel interpolation. We decompose the MSE for a target function into three components - (i) an approximation error, measuring how close the target function is to the span of the representers, (ii) a noiseless estimation error, measuring the error in the absence of noise, and (iii) a noisy estimation error, measuring the average error if the target function is $0$.

We then apply Parseval's Theorem, which relates these errors terms to the Fourier series of the target function, and of the kernel. Proving that the MSE is bounded away from $0$ will rely on our assumptions on the tail and the head of the kernel spectrum.
\section{Decomposition of the mean squared error}

We now derive an exact expression for the MSE as a sum of three error terms:
the approximation error, the noise-free estimation error, and the noisy estimation error. This useful expression will allow us to prove the main theorems of the previous section. 
Recall the definition of $f_\perp^*,\alphavec^*,P_{\X}f^*$ from \Cref{eq:projection}.
\begin{lemma}[MSE Decomposition]
\label{lemma:MSE_decomp}
For any square integrable target function $f^*$, the kernel interpolation $\wh f_N$ estimator satisfies,
\begin{align*}
     \Exp_{\xivec}\MSE{\wh f_N,f^*} &=\! \underbrace{\norm{f^* - P_{\X}f^*}^2 }_{\text{\rm Approximation Error}} 
    + \underbrace{\norm{\inner{\Kmat\inv R_N\curly{f^* - P_{\X}f^*},K(X_N,\cdot)}_N}^2 }_{\text{\rm Noise-free Estimation Error}} \\
    &+ \underbrace{\Exp_{\xivec}{ \norm{\inner{\Kmat\inv\xivec,K(X_N,\cdot)}_N}^2 }}_{\text{\rm Averaged Noisy Estimation Error}}~.
\end{align*}
\end{lemma}
\begin{proof}
Since $P_{\X}f^*-\wh f_N\in\Span\curly{K(x_i,\cdot)}$, the Pythagorean theorem for the triangle $\curly{f^*,P_{\X}f^*,\wh f_N}$, yields,
\begin{align*}
 \MSE{\wh f_N, f^*} &= \norm{f^*-\wh f_N}^2=\underbrace{\norm{f^*-P_{\X}f^*}^2}_{\text{Approximation Error}}+\underbrace{\norm{P_{\X}f^*-\wh f_N}^2}_{\text{Estimation Error}}~.
 \end{align*}

Notice that the estimation error above is random, due to the randomness in $\xivec,$ which affects $\wh f_N$.
Using \Cref{eq:kernel_interpolator_decomposition}, we can further decompose the average estimation error into two error terms,
\begin{subequations}
\begin{align*}
    \Exp_{\xivec}\underbrace{\norm{P_{\X}f^*-\wh f_N}^2}_{\text{Estimation Error}} &= 
    \Exp_{\xivec}\norm{\inner{\Kmat\inv R_N{f_\perp^*}, K(X_N,\cdot)}_N + \inner{\Kmat\inv \xivec,K(X_N,\cdot)}_N}^2~, \\
    &= \underbrace{\norm{\inner{\Kmat\inv R_N{f_\perp^*},K(X_N,\cdot)}_N}^2}_{\text{Noise-free Estimation Error}} + \Exp_{\xivec}\underbrace{\norm{\inner{\Kmat\inv \xivec,K(X_N,\cdot)}_N}^2}_{\text{Noisy Estimation Error}}~.
\end{align*}
\end{subequations}
where the cross term cancels out since the noise is centered. This concludes the proof.
\end{proof}

Computing each of these terms individually, we derive the following expression for the unit circle. Recall the definition of the \textit{$N$-hop subsequences} from \Cref{eq:hop}.
\begin{lemma} For a target function $f^*= \sum_{k\in\Integer}V[k]\phi_k$, we have
\label{lem:MSE_1d}
\begin{enumerate}[label=(\alph*)]
    \item \label{lem:apx}
    Approximation error: 
    \begin{align*}
        \mc E^{\apx}:= \norm{f^* - P_{\X}f^*}^2 = \round{\sum_{i=0}^{N-1} \norm{V_i}^2- \frac{\inner{G_i,V_i}^2}{\norm{G_i}^2}}= \sum_{i=0}^{N-1} \mc E_i^{\apx}~.
    \end{align*}
    \item \label{lem:nfe}
    Noise-free estimation error:
    \begin{align*}
        \mc E^{\free}&:=\norm{\inner{\Kmat\inv R_N f^*_\perp,K(X_N,\cdot)}_N}^2 =  \sum_{i=0}^{N-1} \frac{1}{N}\round{\frac{\inner{V_i,\one}}{\inner{G_i,\one}} - \frac{\inner{G_i,V_i}}{\norm{G_i}^2}}^2 \norm{G_i}^2
        = \sum_{i=0}^{N-1} \mc E_i^{\free}~.
    \end{align*}
    \item \label{lem:nye}
    Averaged noisy estimation error: 
        \begin{align*}
        \mc E^{\noisy}:=\Exp_{\xivec} \norm{\inner{\Kmat\inv\xivec,K(X_N,\cdot)}_N}^2 = \sum_{i=0}^{N-1} \frac{\sigma^2}N \round{\frac{\norm{G_i}}{\norm{G_i}_1}}^2 = \sum_{i=0}^{N-1} \mc E_i^{\noisy}~.
        \end{align*}
\end{enumerate}
Together, this yields that the MSE for the function $f^*$ is,
\begin{align}\label{eq:MSE_full_expression_VG}
    &\Exp_{\xivec}{\MSE{\hat{f}_N,f^*}} 
    = \sum_{i=0}^{N-1} \mc E_i=\sum_{i=0}^{N-1}\mc E_i^\apx + \mc E_i^\free + \mc E_i^\noisy~,\\
    &\mc E_i := \norm{V_i}^2- \frac{\inner{G_i,V_i}^2}{\norm{G_i}^2} + \frac{1}{N} \round{\frac{\inner{V_i,\one}}{\inner{G_i,\one}} - \frac{\inner{G_i,V_i}}{\norm{G_i}^2}}^2 \norm{G_i}^2 +  \frac{\sigma^2}N \round{\frac{\norm{G_i}}{\norm{G_i}_1}}^2~.
\end{align}
\end{lemma}

Appendices \ref{sec:approx_proof}, \ref{sec:noisefree_proof}, and \ref{sec:noisy_proof} provide proofs for \Cref{lem:MSE_1d} \ref{lem:apx}, \ref{lem:nfe}, and \ref{lem:nye} respectively.

\section{Numerical experiments}
\label{sec:expts}
We present experimental results that corroborate our theory. We visualize the effect of kernel bandwidth and regularization on the predictor and test error.
\paragraph{Effect of bandwidth on predictor}
We visualize the effect of bandwidth on kernel interpolator with the Laplace kernel in one dimension (Figure {\ref{fig:interpolation}}). On the $y$-axis we show the predicted values of our estimator (in blue) and the target function $f^*(x) = \cos(x)$ (in orange) with noise level $\sigma^2 = 1$. We notice that for small bandwidth (see $M=2$ plot) kernel interpolation resembles piecewise linear interpolation. Meanwhile, interpolation with high bandwidth converges pointwise (except on a set of measure $0$) to the $0$ function (see $M=200$ plot). Choosing an intermediate bandwidth does not recover the target function either ($M=20$).

\paragraph{Effect of bandwidth on error}We also plot the effect of bandwidth on the exact expected error predicted by our theory (Figure {\ref{fig:predicted_error}}). In this experiment, we study the predicted error of our theory using Laplace kernel interpolation with a noise level $\sigma^2=1$ on a target function $f^*(x) = \cos(x)$. We plot the approximation and noisy estimation errors. We omit the noise-free estimation error as this is typically correlated with approximation error. Our theory predicts that the optimal bandwidth is roughly $M/N = 1$, exactly the point we use to split the cases in the proofs of the main theorems. Interestingly, our theory predicts a trade-off between the approximation error and the error due to noise (noisy estimation error). Larger bandwidths $M$ allow you to fit noise benignly, at the cost of increased approximation error. Smaller bandwidths allow you to approximate well, but suffer in estimation error.

\begin{figure}
\begin{subfigure}{.48\textwidth}
  \centering
\includegraphics[width=\textwidth]{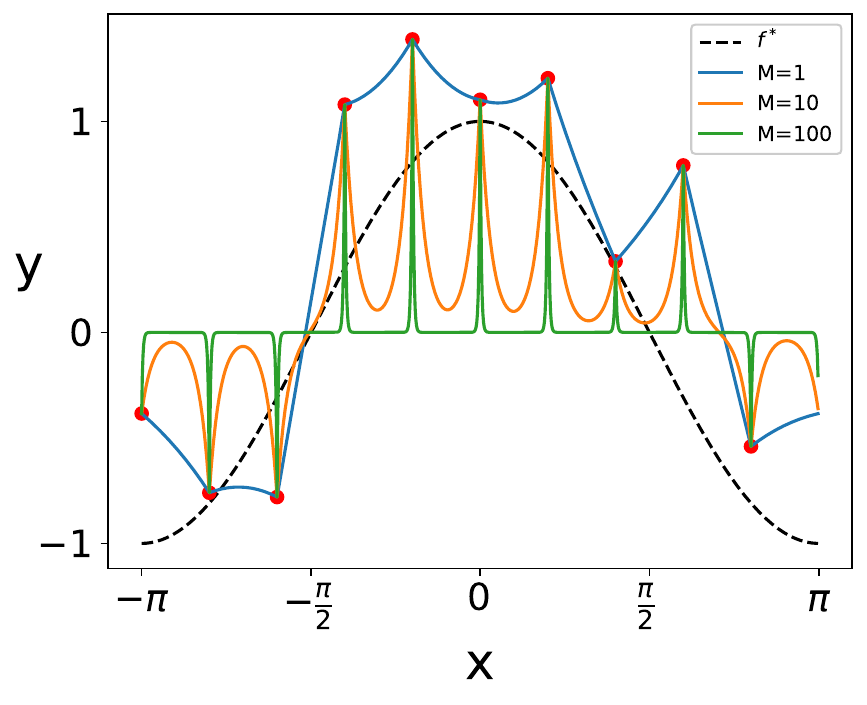}
\subcaption{Effect of kernel bandwidth $M$ on predictor $\wh f_{N,M}$. Here resolution is $N=10,$ noise variance is $ \sigma^2=0.25,$ and target function is $f^*(x)=\cos(x)$.\label{fig:interpolation}}
\end{subfigure}\hfill 
\begin{subfigure}{.48\textwidth}
  \centering
  \includegraphics[width=\textwidth]{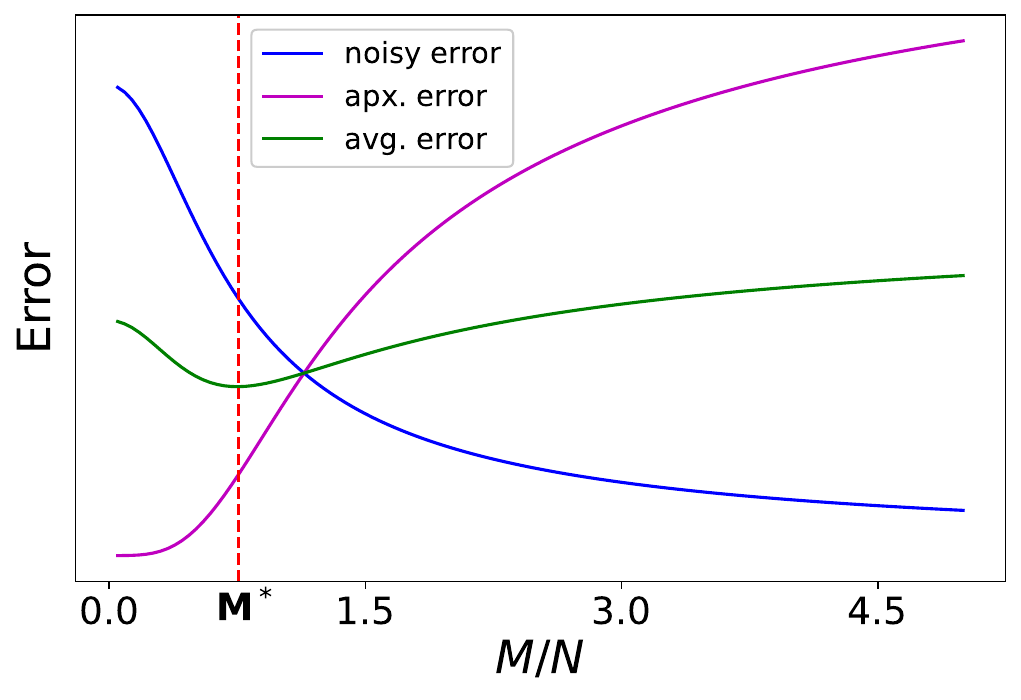}
  \subcaption{Effect of kernel bandwidth on average test error $\MSE{\wh f_{N,M},f^*}$ (in green). See (\Cref{lem:MSE_1d}).\label{fig:predicted_error} for a detailed definition of each error term in the error decomposition. Here resolution  $N=20,$ noise variance $\sigma^2=1,$ and target function is ${f^*=\cos(x)}$.}
\end{subfigure}
\caption{Effect of kernel bandwidth on (\ref{fig:interpolation}) predictor and (\ref{fig:predicted_error}) test error.}
\end{figure}

\paragraph{Effect of regularization} Standard kernel ridge regression (KRR) will prevent interpolation and enable consistent estimation of the target function. However, one can perform interpolation with a modified kernel that mimics regularization to improve generalization while continuing to interpolate. For example, we modify the laplace kernel $K$ on the unit circle to create a new kernel $\wt{K}(x,x') = K(x,x') + \lambda K(M(x-x'))$ for $M=50$ and regularization parameter $\lambda=1$. We compare this modified kernel to Laplace KRR with regularization parameters $\lambda \in \{0,1\}$ in \Cref{fig:modified_reg2}.
\begin{wrapfigure}{r}{0.5\textwidth}
\centering\includegraphics[width=0.47\textwidth]{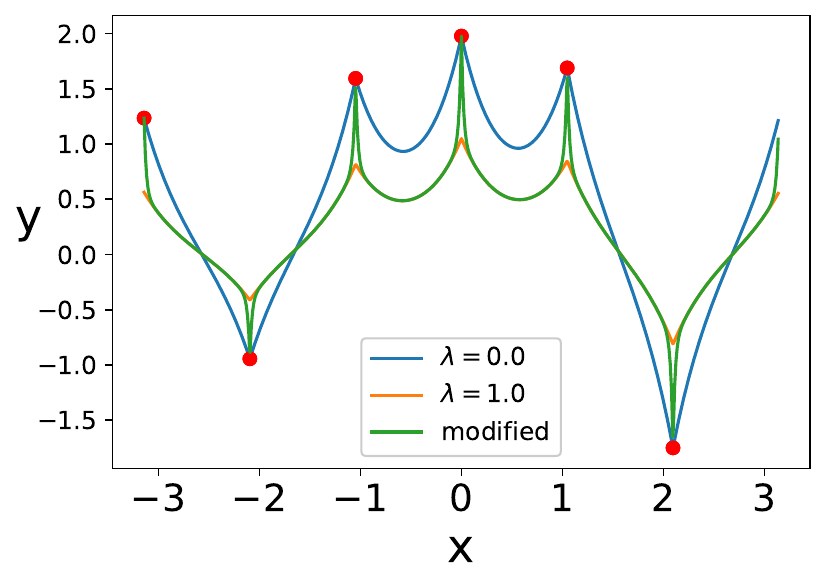}
  \caption{Modified kernel to mimic regularization (for ${f^*(x) = \cos(x)}$).\label{fig:modified_reg2}}
\end{wrapfigure}
\paragraph{Benefits of Regularization}
Our results show that positive regularization allows one to decrease the noisy estimation error at the expense of additional approximation error. Moreover, the faster the decay of the target function's Fourier coefficients, the less regularization worsens the approximation error.
To understand this, we note that adding positive regularization, in effect, adds a Dirac $\delta$-function to the kernel at the origin. In the Fourier domain, this is equivalent to adding an infinitesimal to all of the Fourier coefficients. To understand how this may help generalization, consider adding a small quantity $\Delta>0$ to each of the first $1/\Delta$ Fourier coefficients. We analyze our MSE expression in Lemma {\ref{lem:MSE_1d}}. For a fixed function, adding this $\Delta$ will have a vanishing effect on the approximation error as $\Delta \rightarrow 0$. However, adding this $\Delta$ will decrease the noisy estimation error by extending the tail of $G$, making the ratio $\|G_i\|/\|G_i\|_1$ smaller. We show how a similar modification to the Laplace kernel will cause the interpolated solution to resemble the regularized solution in \Cref{fig:modified_reg2}.

\section{Discussion and Outlook} Following the connection of wide neural networks to kernel methods \cite{JacotNTK}, the theory of kernel methods has seen a renewed interest as a tool to better understand deep neural networks \cite{belkin2018understand}. Kernel methods, being analytically more tractable than neural networks, can yield significant insights about the behavior of deep networks. However several questions remain unanswered about the behavior of kernel methods themselves.

In this paper, we investigated 
the consistency of kernel methods in fixed dimensions.
We showed that kernel interpolation, or kernel ridgeless regression, is inconsistent in fixed dimension even with adaptive bandwidth. This provides a generalization of the main result in~\cite{rakhlin2019consistency}, which considered the special case of the Laplace kernel,
to a broad class of translation-invariant kernels including the Gaussian, Lapalce, and Cauchy kernels. 

Our work suggests that infinitely-wide neural networks are inconsistent in fixed dimensions, as these networks are equivalent to kernel machines \cite{JacotNTK}. It is an interesting direction for future work if feature learning in finite-width networks \cite{RFM_paper} can enable consistency.

Further, while our result may be perceived as a negative result about kernel methods, it still leaves open the possibility of bounded inconsistency under interpolation, also called \textit{tempered overfitting} in \cite{mallinar2022benign}. It remains unclear when interpolation may be an acceptable solution concept. In any case, consistency can be enabled using appropriate regularization.

\paragraph{The Role of Data Dimension}
When the dimension of the inputs scales with the number of samples, kernel ridgeless regression can generalize \cite{liang2020just}. Our results provide additional evidence that high dimensions can dissipate the error due to noise. In particular, under our assumptions on the kernel, for the expression of the noisy estimation error (\Cref{lem:MSE_highd}\ref{lem:nye_highd}), the constants decay exponentially with dimension. This dependence was also observed in \cite{rakhlin2019consistency} for the Laplace kernel. As an additional effect, for target functions with norm that is invariant to dimension (before scaling by $(2\pi)^{-d}$), the $0$-estimator has approximation error that vanishes exponentially with dimension.
Further, to counteract the error due to noise, the bandwidth should be much larger than the data resolution in each dimension, \ie, $M>N$. However, when the dimensions grow with the number of samples, say $d=\omega(\log n)$, the resolution $N = n^{1/d} = O(1)$ in each dimension is approximately constant, and therefore the bandwidth does not need to increase with $d,n$ to satisfy $M>N$. As increasing the bandwidth in general will worsen the approximation error, the constancy of the bandwidth is a form of the blessing of dimensionality.

\section*{Acknowledgements}
We are grateful for support of the NSF, the Simons Institute for the Theory of Computing, and the Simons Foundation for the Collaboration on the Theoretical Foundations of Deep Learning\footnote{\url{https://deepfoundations.ai/}} through awards DMS-2031883 and \#814639. We also acknowledge NSF support through IIS-1815697 and the TILOS institute (NSF CCF-2112665). We also thank Sam Buchanan, Jamie Simon, Jonathan Shi, and the anonymous reviewers for useful conversations, questions, and feedback on this work.

\bibliographystyle{plain}
\bibliography{aux/ref} 

\appendix
\vspace{.25in}
\begin{center}
    \underline{\huge\sf Appendices}
\end{center}
\vspace{.1in}

\section{Proof of main result (\texorpdfstring{$d=1$}{d=1})}
\label{sec:proofs}

We prove the main results. The proof strategy is to obtain an $\Omega(1)$ lower bound on $\Exp_{\xivec}\MSE{\wh f_N,f^*}$. \Cref{eq:MSE_full_expression_VG} expressed this quantity as a sum of $N$ non-negative quantities. We show that at least $\Omega(N)$ of these quantities, $\mc E_i$, are $\Omega(1/N).$

\paragraph{Proof of \texorpdfstring{\Cref{theorem:inconsistent_all_band}}{Theorem 1}}
When $M>N$, we show that the approximation error is large in the base RKHS $\Hilbert_0$.
On the other hand, when $M\leq N$ we show that the averaged noisy estimation error has a constant lower bound.

\subparagraph{Case 1, $M \leq N$} In this case we show that the noisy estimation error is bounded away from $0$. Define,
$
\displaystyle    \Delta_{i} := \|G_{i}\|_1 - \abs{G[i]}=\sum_{m\neq 0}\abs{G[mN+i]} \geq 0
$.
Assumption \ref{ass:tail} says that for all but $o(N)$ terms corresponding to $i \in [N]$, we have,
\begin{align}
    \Delta_{i} = \sum_{m\neq0}\abs{G[mN+i]}\leq C_1\abs{G[i]}\sum_{m\neq0}\frac{1}{1+m^2}\leq4C_1\abs{G[i]}~. \label{eq:Delta_lower_bound}
\end{align} 
For such an $i$, we can lower bound the noisy estimation error term $\mc E_i^\noisy$ as,
\begin{align*}
    \mc E_i^\noisy&=\frac{\sigma^2}N\frac{\norm{G_i}^2}{\norm{G_i}_1^2}=\frac{\sigma^2}N\frac{\norm{G_i}^2}{\rbrac{\abs{G[i]} + \Delta_{i}}^2 } 
    \geq 
    \frac{\sigma^2}N\frac{\abs{G[i]}^2}{2\abs{G[i]}^2 + 2\abs{\Delta_i}^2 }
    \geq \frac{\sigma^2}{2N(1+4C_1)}\\
    \mc E^{\noisy} &= \sum_{i=1}^N\mc E^\noisy_i=\Omega(\sigma^2),
\end{align*}
since there are $\Omega(N)$ such indices $i \in [N]$ for which \Cref{eq:Delta_lower_bound} holds.

\subparagraph{Case 2, $M > N$} In this case we show the approximation error will be bounded from $0$.
Since $M > N$, by Assumption \ref{ass:head}, there exists a fixed integer $i^*$, such that $\abs{G[i^*]} \leq C_3\abs{G[N + i^*]}$. Now let $f^*(x)=\sqrt{2}\cos(i^*x)$ be the (real-valued) function with Fourier coefficients $V[i^*] = V[-i^*]= \frac{1}{\sqrt{2}}$, and $V[k] = 0$ for $|k| \neq i^*$. Using this, we can lower bound the approximation error as,
\begin{align*}
    \mc E^\apx\geq \mc E_{i^*}^\apx \geq V[i^*]^2\rbrac{1 - \frac{\frac12\abs{G[i^*]}^2}{\sum_{m \in \Integer} \abs{G[mN + i^*]}^2}} 
    &\geq  \frac12\rbrac{1 - \frac{\frac12\abs{G[i^*]}^2}{\abs{G[N+i^*]}^2 + \abs{G[i^*]}^2}} \\
    &\geq \frac{1+2C_3^2}{2+2C_3^{2}}=\Omega(1)~.
\end{align*}
The fact that $G_0[i^*]>0$ from the Assumption \ref{ass:head}, also allows us to conclude that, \begin{align*}
    \norm{f^*}_{\Hilbert_0}^2 = \frac{|V[i^*]|^2}{G_0[i^*]}+\frac{|V[-i^*]|^2}{G_0[-i^*]}<\infty.
\end{align*}

\begin{proof}[Proof of \Cref{theorem:inconsistent_all_f}]
This follows from the proof of \Cref{theorem:inconsistent_all_band}. We showed that for $M\leq N$ (Case 1 above), the noisy estimation error $\mc E^\noisy$ satisfies $\mc E^\noisy=\Omega(\sigma^2)$. Since $\mc E^\noisy$ does not involve the target function $f^*$, the statement of \Cref{theorem:inconsistent_all_f} follows.
\end{proof}

\begin{proof}[Proof of \Cref{thm:inconsistency_monotonic}]
For $M\leq N$, we know from Case 1 in the proof of \Cref{theorem:inconsistent_all_band} that $\mc E^\noisy=\Omega(1).$ For $M>N,$ we will show that $\mc E^\apx=\Omega(1)$ which sufficies to prove the claim.

Note we have $G[k]=G[-k]$. We start by the monotonic boundedness of $G$, we have
\begin{align*}
    \frac{|G[i]|^2}{\norm{G_i}^2}=\frac{|G[i]|^2}{\abs{G[i]}^2+\sum_{k\in\Integer*}\abs{G[i+kN]}^2} \leq \frac{1}{1+\upsilon}
\end{align*}
for $|i|<N/2$, where $\upsilon=C^2(C'')^2/c^2 > 0$.  

Now consider $f^* = \sum_{k \in \Integer} V[k] \phi_k$, and suppose $\|f^*\|^2=\sum_{k\in\Integer}\abs{V[k]}^2=1$. As $f^*$ has square-integrable zeroth and first derivatives, its Fourier series coefficients have decay $|V[i]| \leq B/(1+i^2)$ for all $i$, for a sufficiently large constant $B$. This implies for all $N>0$,
\begin{align}
    \sum_{|i| > N/2} |V[i]| &< \frac{4B}{N},\qquad\text{and}\label{eq:ub_V}\\
    \sum_{|i| < N/2}|V[i]|^2&\geq 1-\frac{48B^2}{N^3}.\label{eq:lb_V}
\end{align}
One can show by contradiction (to equation \eqref{eq:lb_V}) -- for some $\eps\in[0,\frac12]$ there exists a subset $S_\eps\subset \set{i:\abs{i}<N/2}$ such that $|S_\eps|=\Omega(N^{2\eps})$ and $|V[i]|\geq \Omega(N^{-\eps})$ for all $i\in S_\eps$.
For any such $\eps$, consider the following set of inequalities for $i\in S_\eps$,
\begin{align*}
    \mc E^\apx_i &=\norm{V_i}^2 - \frac{\inner{G_i,V_i}^2}{\norm{G_i}^2}\\
    &\geq \norm{V_i}^2 - \frac{|G[i]|^2}{\norm{G_i}^2}\round{\sum_{k\in\Integer}|V[i+kN]|}^2\geq \norm{V_i}^2 - \frac{1}{1+\upsilon}\round{\sum_{k\in\Integer}|V[i+kN]|}^2\\&\geq |V[i]|^2 - \frac{1}{1+\upsilon}\round{\sum_{k\in\Integer}|V[i+kN]|}^2
    =|V[i]|^2- \frac{1}{1+\upsilon}\round{|V[i]|+\sum_{k\in\Integer*}|V[i+kN]|}^2\\ 
    &\geq|V[i]|^2- \frac{1}{1+\upsilon}\round{|V[i]|+\frac{4B}{N}}^2
    = |V[i]|^2\round{1 - \frac{1}{1+\upsilon}\round{1+\frac{4B}{N|V[i]|}}^2} \\
    &\geq \Omega(N^{-2\eps})\round{1 - \frac{\round{1+4B\cdot O(N^{\eps-1})}^2}{1+\upsilon}}
\end{align*}
Since $\eps -1 < 0$ due to the range of $\eps$, the term in the inner parenthesis always approaches 1 for large enough $N.$ Hence we have,
\begin{align}
    \mc E^\apx =\sum_{i=1}^N \mc E^\apx_i \geq \sum_{\substack{i\in S_\eps\\ i> 0}}^N \mc E^\apx_i + 
    \sum_{\substack{i\in S_\eps\\ i<0}}^N \mc E^\apx_{N-i} \geq \Omega(N^{-2\eps})\cdot|S_\eps|=\Omega(1).
\end{align}
This proves the claim.
\end{proof}

\section{Decomposition of MSE: Proof of \texorpdfstring{\Cref{lem:MSE_1d}}{}}

Appendices \ref{sec:approx_proof}, \ref{sec:noisefree_proof}, and \ref{sec:noisy_proof} provide proofs for \Cref{lem:MSE_1d} \ref{lem:apx}, \ref{lem:nfe}, and \ref{lem:nye} respectively. Recall that $P_\X$ is the $L^2$ projection operator onto $\text{span}\round{\set{K(x_i,\cdot)}}$

\subsection{Approximation error: Proof of \texorpdfstring{\Cref{lem:MSE_1d}\ref{lem:apx}}{}}
\label{sec:approx_proof}

The proof proceeds by applying the Pythagorean theorem to the triangle $\curly{0,f^*,P_{\X}f^*}$ in $L^2_\mu$. The following lemma gives exact expressions for projection of the target function and its norm. 

\begin{lemma}[Projection] For $f^*=\sum_{k\in\Integer}V[k]\phi_k$
\label{lemma:projection}
\begin{align*}
    P_{\X}f^* = \sum_{\ell=0}^{N-1} \frac{\inner{V_\ell,G_\ell}}{{\norm{G_\ell}^2}} \sum_{m \in \Integer} G[mN + \ell] \phi_{mN + \ell},\qquad\text{and}\qquad
    \norm{P_{\X}f^*}^2 = \sum_{\ell=0}^{N-1} \frac{\abs{\inner{V_\ell,G_\ell}}^2}{\norm{G_\ell}^2}
\end{align*}
\end{lemma}

We get,
\begin{align*}
    \norm{f^* - P_{\X}f^*}^2 
    = \norm{f^*}^2 - \norm{P_{\X}f^*}^2 
    = \norm{V}^2 - \sum_{\ell=0}^{N-1} \frac{\inner{V_\ell,G_\ell}^2}{\norm{G_\ell}^2}
\end{align*}

\begin{proof}[Proof of Lemma \ref{lemma:projection}]
Note that \Cref{lem:eigenfunctions_of_Kcov} shows that $\curly{\frac{\psi_\ell}{\|\psi_\ell\|}}_{\ell=0}^{N-1}$ is an orthonormal basis for $\Span\{K(x_{\ell},\cdot)\}$. Consequently, we have \begin{align*}
    P_{\X}f^* = \sum_{\ell=0}^{N-1} \inner{f^*, \frac{\psi_\ell}{\|\psi_\ell\|}} \frac{\psi_\ell}{\|\psi_\ell\|},\qquad\text{and}\qquad
    \norm{P_{\X}f^*}^2 = \sum_{\ell=0}^{N-1}\inner{f^*, \frac{\psi_\ell}{\|\psi_\ell\|}}^2
\end{align*} 
We compute these projections below. For $\ell \in [N]$,
\begin{align*}
    \sqrt{\norm{G_\ell}_1}\inner{f^*, \psi_\ell} &=  \inner{ \sum_{k \in \Integer} V[k] \phi_k , \sum_{m \in \Integer} G[mN + \ell]\phi_{mN + \ell}} \\
    &=  \sum_{m,k \in \Integer} G[mN + \ell] V[k] \indicator{k=mN + \ell }\\
    &=  \sum_{m \in \Integer} G[mN + \ell] V[mN + \ell]=\inner{V_\ell,G_\ell}
\end{align*}
Thus, we get that,
\begin{align*}
    \inner{f^*, \frac{\psi_\ell}{\|\psi_\ell\|}} \frac{\psi_\ell}{\|\psi_\ell\|} = \frac{\inner{V_\ell,G_\ell}}{\norm{G_\ell}^2}\sum_{m\in\Integer}G[mN+\ell]\phi_{mN+\ell}
\end{align*}
The claims follow immediately.
\end{proof}

\subsection{Noise-free estimation error: Proof of \texorpdfstring{\Cref{lem:MSE_1d}\ref{lem:nfe}}{}}
\label{sec:noisefree_proof}

Let $E$ be the fourier series of $\inner{\Kmat\inv R_N\curly{f^*-P_{\X}f^*},K(X_N,\cdot)}$. From \Cref{lemma:coeff_inspan} we have,
\begin{align*}
    E[k] = \sqrt{N}R_N \cbrac{f^* - P_{\X}f^*}\tran\Kmat\inv\u_{k\mod N} \cdot G[k]
\end{align*}
By Parseval's theorem (\Cref{thm:parseval_continuous}), we conclude,
\begin{align*}
    \frac{1}{2\pi} \int_{-\pi}^\pi \rbrac{\inner{\K\inv R_N \cbrac{f^* - P_{\X}f^*},K(X_N,t)}_N}^2\dif t = \sum_{k \in \Integer} |E[k]|^2
\end{align*}
First, we show that
\begin{align*}
    {R_N\cbrac{f^* - P_{\X}f^*}\tran \K\inv\u_{\ell}}
    &= \rbrac{\frac{\inner{V_i,\one}}{\inner{G_i,\one}} - \frac{\inner{G_\ell , V_\ell}}{\|G_\ell\|^2}}
\end{align*}

From \Cref{prop:eigenvectors_of_Kmat} we have $\Kmat\inv\u_{\ell} = \u_{\ell}\cdot \frac{1}{N\norm{G_{\ell}}_1}$.
Thus by \Cref{lemma:projection}, we can write $P_{\X}f^*$ on the data as
\begin{align*}
    P_{\X}f^*(x_i) &= \sum_{\ell=0}^{N-1} \frac{\inner{G_\ell,V_\ell}}{\norm{G_\ell}^2} \sum_{m\in\Integer}G[mN+\ell] \phi_{mN+\ell}(x_i)= \sum_{\ell=0}^{N-1} \frac{\inner{G_\ell,V_\ell}}{\norm{G_\ell}^2} \norm{G_\ell}_1 \wb{u_{\ell i}}\\
    f^*(x_i)&=\sum_{k\in\Integer}V[k]\phi_k(x_i)=\sum_{\ell=0}^{N-1}\inner{V_i,\one}\wb{u_{\ell i}}
\end{align*}
We thus have
\begin{align*}
    R_N \cbrac{f^* - P_{\X}f^*}\tran\Kmat\inv\u_{\ell} &=  \sum_{i,\ell'=0}^{N-1}\round{\inner{V_{\ell'},\one}-\frac{\inner{G_{\ell'},V_{\ell'}}}{\norm{G_{\ell'}}^2}\norm{G_{\ell'}}_1} \frac{\wb{u_{\ell' i}}u_{\ell i}}{N\norm{G_\ell}_1} \\
    &=
    \frac{1}{N}\round{\frac{\inner{V_{\ell},\one}}{\inner{G_{\ell},\one}}-\frac{\inner{V_\ell,G_\ell}}{\norm{G_\ell}^2}}
\end{align*}
This gives,
\begin{align*}
    \sum_{k\in\Integer}\abs{E[k]}^2=\frac{1}{N}\sum_{\ell=0}^{N-1}\round{\frac{\inner{V_{\ell},\one}}{\inner{G_{\ell},\one}}-\frac{\inner{V_\ell,G_\ell}}{\norm{G_\ell}^2}}^2\norm{G_\ell}^2.
\end{align*}

\subsection{Noisy estimation error: Proof of \texorpdfstring{\Cref{lem:MSE_1d}\ref{lem:nye}}{}}
\label{sec:noisy_proof}

We derive this by an application of Parseval's theorem.
Define the Fourier series,
\begin{align*}
    \inner{\Kmat\inv\xivec,K(X_N,t)}_N = \sum_{k\in\Integer} \wt E[k] e^{j k t}
\end{align*}
By \Cref{thm:parseval_continuous} (Parseval's theorem), we have, 
\begin{subequations}
\begin{align*}
    &\Exp_{\xivec} \sbrac{ \frac{1}{2\pi} \int_{-\pi}^{\pi} |\inner{\Kmat\inv\xivec,K(X_N,t)}_N|^2 \dif t } = \sum_{k\in\Integer} \Exp_{\xivec} |\wt E[k]|^2 
    = \sum_{i=0}^{N-1} \sum_{m \in \Integer} \Exp_{\xivec} \abs{\wt E[mN + i]}^2 \\
    &\overset{\rm (a)}= \sum_{i=0}^{N-1} \sum_{m \in \Integer} \left|G[mN+i]\right|^2 \Exp_{\xivec} \left|\xivec\tran\Kmat\inv\u_i \right|^2\cdot N 
    \overset{\rm (b)}= \sigma^2 \sum_{i=0}^{N-1} \norm{G_i}^2  \left( \u_i\herm\Kmat^{-2} {\u}_i  \right)\cdot N \\
    &\overset{\rm (c)}= {\sigma^2} \sum_{i=0}^{N-1}\norm{G_i}^2\frac{1}{N^2\norm{G_i}_1^2} N = {\sigma^2} \sum_{i=0}^{N-1}\frac{\norm{G_i}^2}{N\norm{G_i}_1^2}
\end{align*}
\end{subequations}
where we have used \Cref{lemma:coeff_inspan} in (a), and \Cref{lemma:noise_dft} in (b), and \Cref{prop:eigenvectors_of_Kmat} in (c).
\section{Main Results for \texorpdfstring{$d>1$}{d>1}}
\label{sec:main_results_highd}

We can perform a similar analysis for $d>1.$ To generalize the main results, we also generalize Assumptions \ref{ass:scale}-\ref{ass:head} for the kernel to Assumptions \ref{ass:scale_highd}-\ref{ass:head_highd}
for dimensions greater than one, as well as the monotonicity condition (\Cref{con:monotonic_highd}). Under these assumptions, we show the main results hold.

\begin{theorem}[Inconsistency for all Functions (when $G$ is monotonically bounded)]
\label{thm:inconsistency_monotonic_highd}
Consider a fixed non-zero regression function $f^*$ (i) in the Sobolev space of order $\frac{d}{2}+1$ (i.e. whose derivatives of orders $\alpha \in \Integer_{\geq 0}^d$ for $\|\alpha\|_1 \leq \frac{d}{2}+1$ are square-integrable) and (ii) that can be expressed as a convergent Fourier series. Then, interpolation with 
a real-valued translation-invariant kernel satisfying \Cref{con:monotonic_highd} 
is inconsistent for $f^*$, for any bandwidth, even if chosen adaptively.
\end{theorem}

See \cite{lieb2001analysis} for a definition of an order $\alpha$ derivative. 

\begin{theorem}[Inconsistency for all Bandwidths]
\label{theorem:inconsistent_all_band_highd}
For any translation-invariant kernel satisfying Assumptions \ref{ass:scale_highd}-\ref{ass:head_highd}, there exists a function with constant $\mathcal{H}_0$-norm for which kernel interpolation will be inconsistent for any bandwidth, even adaptive to the data set.
\end{theorem}

\begin{theorem}[Inconsistency for all Functions]
\label{theorem:inconsistent_all_f_highd}
For any translation-invariant kernel satisfying Assumptions \ref{ass:scale_highd}-\ref{ass:tail_highd}, with a bandwidth $M \leq N$, kernel interpolation will be inconsistent for all target functions that can be expressed as  convergent Fourier series. In particular, kernel interpolation with any fixed bandwidth will be inconsistent for all such functions.
\end{theorem}

Further, these results hold for the Gaussian, Laplace, and Cauchy kernels. 

\begin{proposition}
\label{lemma:ex_kernels_inconsistent_highd}
The Gaussian $G[k] = \exp(-\frac{\|\k\|^2}{M^2})$, Laplace $G[\k] = \round{1 + \frac{\|\k\|^2}{M^2}}^{-\frac{d+1}{2}}$, and Cauchy $G[\k] = \exp(-\frac{\|\k\|}{M})$ kernels (wrapped on the unit circle) satisfy \Cref{con:monotonic_highd}.  
\end{proposition}

\newpage
\begin{center}
    {\sf \underline{\huge Supplementary materials}: \\[1ex]
    \Large On the Inconsistency of Kernel Ridgeless Regression in Fixed Dimensions}
\end{center}\vspace{5ex}
\section{Extending proofs to higher dimensions}
\label{app:highd}
Proofs missing from this section are provided in the supplementary materials.
\paragraph{Notation} In this section $d\geq 1$ and $n=N^d$. By $[N]^d$ we denote the $d$-fold Cartesian product of $[N]:=\set{0,1,\ldots,N-1}$. For vectors $\p,\q$, we write $\p \leq \q$ to indicate a coordinate-wise inequality, \ie, for all coordinates $i$, we have $p_i \leq q_i$. Similarly, $\p \nleq \q$ indicates $\p\leq\q$ is violated, \ie, there exists a coordinate $i$ for which $p_i > q_i$. We similarly define $\p \geq \q$ and $\p \ngeq \q.$ We also denote $\zero$ and $\one$ to be the vectors of all 0's and all 1's respectively in a dimension compatible with the expression. For a scalar $C$, the expression $\p \leq C$ means $\p\leq C\cdot\one,$ and similarly $\p \geq C, \p \nleq C, \p \ngeq C$.

We consider sequences indexed by $\Integer^d$, and for such sequences we extend the definition of $N-$ \textit{hop subsequences} from \Cref{eq:hop} in the following manner. For a fixed $N\in\Natural,$ and a sequence $G\in\ell^1(\Integer^d)$, and for $\l\in[N]^d$, let  
\begin{align*}
G_\l \in\ell^1(\Integer^d)\qquad G_{\l}[\m] = G[\m N + \l],\qquad \forall\,\m\in\Integer^d
\end{align*}
be the \textit{$N$-hop subsequence} with entries given as above.
For $\k \in \Integer^d$, and $\x\in\Real^d$
\begin{align}\label{eq:mod_highd}
    \k \mod N &:= (k_1 \text{ (mod } N),~ k_2 \text{ (mod } N),~ \ldots,~ k_d \text{ (mod } N)) \in [N]^d\\
    \x \mod \unitcircle &:= (x_1 \text{ (mod } \unitcircle),~ x_2 \text{ (mod } \unitcircle),~ \ldots,~ x_d \text{ (mod } \unitcircle)) \in \unitcircle^d
\end{align}
where we remind the reader of notation from \Cref{eq:def:mod_unitcircle}.
We denote by $\unitcircle^d$ the Cartesian product of $d$ unit circles $\unitcircle$ along each dimension. We refer to this as the unit torus.

\begin{definition}[Fourier basis]
For $k\in\Integer^d$ and $\x\in\unitcircle^d$, define $\phi_\k(\x) :=  \exp\round{j\inner{\k,\x}}=\prod_{i=1}^d \exp(jk_ix_i)$. This basis satisfies $\inner{\phi_\k,\phi_\l}=\frac1{(2\pi)^d}\int_{\unitcircle^d}\exp(j\inner{\k-\l,\x})\dif\x=\indicator{\k=\l}$.
\end{definition}

A target function defined on the unit torus admits a Fourier series, 
\begin{align}\label{eq:target_fn_highd}
f^* = \sum_{\k \in \mathbb{Z}^d} V[\k] \phi_{\k}    \qquad V[\k]=\inner{f^*,\phi_\k}.
\end{align}

\begin{definition}[DFT Matrix \texorpdfstring{$d>1$}{d>1}]\label{def:dft_d>1}
The normalized DFT matrix in $d>1$ is 
\begin{align*}
\U_d = \begin{bmatrix}
    \u_{\zero}   & \cdots & \u_{(N-1)\one}
\end{bmatrix}\in\Complex^{N^d\times N^d},
\qquad 
\u_{\l,\p}  := N^{-\nicefrac{d}{2}} \exp\round{-j \frac{2\pi}N\inner{\l,\p}},\quad \l,\p \in [N]^d
\end{align*}
\end{definition}

\paragraph{Data distribution}
For $d>1$, the continuous distribution is $\mu=\text{Uniform}(\unitcircle^d)$ and the discrete distribution over $\x\in\unitcircle^d$, with $n=N^d$ samples, to be
\begin{align*}
    \mu_n(\x):=\frac1{N^d}\sum_{\ell\in[N]^d}\delta(\x-\x_\l),\qquad (\x_\l)_i = \frac{2\pi}{N}\ell_i - \pi,\qquad \forall\,\l\in[N]^d, \text{ and } \forall\, i\in[N].
\end{align*}
The $n=N^d$ samples $\{\x_\l\}$ are indexed by elements of $[N]^d$, where $\x_\l\in\unitcircle^d$ and has coordinates given by the expression above. Note again that $\mu_{n}$ weakly converges to $\mu.$ In the rest of this section we use
\begin{align*}
    \inner{\alphavec,K(X_n, \cdot)}_n := \sum_{\p \in [N]^d} \alpha_\p K(\x_\p,\cdot)
\end{align*}
to keep the notation simple.

\paragraph{Translation-invariant kernels}

As in $d=1$, we can define translation-invariant kernels with the following property,
\begin{align*}
    K(\x,\x') = g\round{M(\x-\x' \mod \unitcircle)},\qquad \x,\x' \in \unitcircle^d
\end{align*}
for some even function $g: \mathbb{R}^d \rightarrow \mathbb{R}$ (see Definition \eqref{eq:mod_highd}). 

Define $G_0,G:\Integer^d \rightarrow \Complex$ as the Fourier series, \ie
\begin{subequations}
\begin{align}\label{eq:FS_of_g_highd}
    &g(M(\bm{\theta} \mod \unitcircle)) = \sum_{\k \in\Integer^d} G[\k] \exp(j \inner{\k,\bm{\theta}})  \\
    &G[\k] = \frac{1}{(2\pi)^d}\int_{\unitcircle^d} g(M\bm{\theta})\exp(-j \inner{\k,\bm{\theta}})\dif \bm{\theta}
\end{align}
\end{subequations}
Note that while usually the bandwidth scales the input (as detailed here), our analysis also holds for different mechanisms.

As in $d=1$, for positive definite kernels, $g$ is an even function whereby we have,
\begin{align*}
    G[\k] = G[-\k]\geq 0 \quad \forall \k \in \Integer^d
\end{align*}

\subsection{MSE Decomposition}
We start with a result analogous to \Cref{lemma:MSE_decomp}, when $d>1.$
\begin{lemma}[Decomposition of MSE for \texorpdfstring{$d>1$}{d>1}] For a target function $f^*= \sum_{\k \in\Integer^d}V[\k]\phi_{\k}$,
\label{lem:MSE_highd}
\begin{enumerate}[label=(\alph*)]
    \item \label{lem:apx_highd}
    Approximation error: 
    \begin{align*}
        \displaystyle\mc E^{\apx}:= \norm{f^* - P_{\X}f^*}^2 = \sum_{\p \in [N]^d} \norm{V_\p}^2- \frac{\inner{G_\p,V_\p}^2}{\norm{G_\p}^2} = \sum_{\p \in [N]^d} \mc E_\p^{\apx}
    \end{align*}
    \item \label{lem:nfe_highd}
    Noise-free estimation error:
    \begin{align*}
        \mc E^{\free}&:=\norm{\inner{\Kmat\inv R_n\curly{f^* - P_{\X}f^*},K(X_n,\cdot)}_n}^2 =  \sum_{\p \in [N]^d} \frac{\norm{G_\p}^2}{N^d}\round{\frac{\inner{V_{\p},\one}}{\inner{G_{\p},\one}} - \frac{\inner{G_\p,V_\p}}{\norm{G_\p}^2}}^2 \\
        &= \sum_{\p \in [N]^d} \mc E_\p^{\free}
    \end{align*}
    \item \label{lem:nye_highd}
    Averaged noisy estimation error: 
    \begin{align*}
        \mc E^{\noisy}:=\Exp_{\xivec} \norm{\inner{\Kmat\inv\xivec,K(X_n,\cdot)}_n}^2 = \sum_{\p \in [N]^d} \frac{\sigma^2}{N^d} \round{\frac{\norm{G_\p}}{\norm{G_\p}_1}}^2 = \sum_{\p \in [N]^d} \mc E_\p^{\noisy}
    \end{align*}
        
\end{enumerate}
Together, this yields that the MSE for this function is,
\begin{align}\label{eq:MSE_full_expression_VG_highd}
    \Exp_{\xivec}{\MSE{\hat{f}_N,f^*}} 
    &= \sum_{\p \in [N]^d} \mc E_\p=\sum_{\p \in [N]^d} \mc E_\p^\apx + \mc E_\p^\free + \mc E_\p^\noisy\\
    \mc E_\p := \norm{V_\p}^2 &- \frac{\inner{G_\p,V_\p}^2}{\norm{G_\p}^2} + \frac{\norm{G_\p}^2}{N^d} \round{\frac{\inner{V_{\p},\one}}{\inner{G_{\p},\one}} - \frac{\inner{G_\p,V_\p}}{\norm{G_\p}^2}}^2 
    +  \frac{\sigma^2}{N^d} \round{\frac{\norm{G_\p}}{\norm{G_\p}_1}}^2
\end{align}
\end{lemma}

The derivation for $d>1$ is similar to the $d=1$ case. 
\Cref{sec:approx_proof_highd}, \Cref{sec:noisefree_proof_highd}, and \Cref{sec:noisy_proof_highd} in the supplementary materials provide proofs for \Cref{lem:MSE_highd} \Cref{lem:apx_highd}, \Cref{lem:nfe_highd}, and \Cref{lem:nye_highd} respectively.

\subsection{Spectral Assumptions}

We assume spectral conditions for kernels in $d>1$ that are analogous to those in $d=1$. 

\begin{assumption}[Integrability]
\label{ass:scale_highd}
We assume the kernel is integrable. In particular, the integral ${\int_{\unitcircle^d} g(M\x) \dif \x}$ exists and is finite for all $0<M< \infty$.
\end{assumption}

\begin{assumption}[Spectral Tail]
\label{ass:tail_highd}
For all $\k \in \Integer^d_{\geq 0}$, there exists a dimension-dependent constant $C_{1,d} > 0$ such that,
\begin{align}\label{eq:tail_assumption_highd}
    |G[\k M^\prime  + \p]| \leq C_{1,d} |G[\p]| \round{1 + \|\k\|^2}^{-\frac{d+1}{2}}
\end{align}
holds for all $M' \geq M$ and for all $0 \leq \p \leq M'$, except $o_{M'}((M')^d)$ many.
\end{assumption}

\begin{assumption}[Spectral Head]
\label{ass:head_highd}
There exist dimension-dependent constants $C_{2,d}, C_{3,d}\in\Real_+$, $\p^*\in\Integer_{\geq 0}$, and $\zero\leq \m^* \leq \one$ with $\m^* \neq 0$, such that for $M\geq C_{2,d}$, we have that for all $M'\leq M$, $|G[\p^*]| \leq C_{3,d} |G[\p^* + M'\m^*]|$  and $\abs{G_0[\p^*]}>0$.
\end{assumption}

We give a condition that implies all three of these assumptions:
\begin{condition}[Monotonic Boundedness ($d>1$)]
\label{con:monotonic_highd}
There exist constants (that are independent of the bandwidth $M$) $c_d,C_d,C'_d,C''_d>0$, a constant $c'_d(M)$ (that may depend on $M$) and a monotonically decreasing function $f(\|\x\|)$ with (i) $0 < \frac{f(\|\x+\k\|)}{f(\|\x\|)} \leq C''_d \round{1 + \|\k\|^2}^{-\frac{d+1}{2}}$ for all $x \in \Real^d_{\geq 0}$,$\k \in \Integer_{\geq 0}^d$, and (ii) $\frac{f(\|\e_i + \x\|)}{f(\|\x\|)} \geq C'_d$ for $\zero \leq \x \leq \one$ and standard basis vectors $\e_i$ for $i \in [d]$, such that $c_d f(\|\k\|/M) \leq G[\k]/c'_d(M) \leq C_d f(\|\k\|/M)$ for all $\k \in \Integer^d_{\geq0}$.
\end{condition}

With the assumptions defined, we can prove the main results for $d>1$.

\subsection{Proof of \texorpdfstring{\Cref{theorem:inconsistent_all_band_highd}}{}}

When $M\leq N$ we show that the averaged noisy estimation error is large.
On the other hand, when $M>N$, we show that the approximation error is large for a cosine function in the base RKHS $\Hilbert_0$.

\paragraph{Case 1, \texorpdfstring{$M > N$}{M>N}} In this case we show the approximation error is bounded away from $0$.
Since $M > N$, by \Cref{ass:head_highd}, there exists a vector $\p^*$ of constant integers, and an $\zero \leq \m^* \leq \one$ with $\m^* \neq 0$, such that $\abs{G[\p^*]} \leq C_{3,d} \abs{G[\m^* N + \p^*]}$. Now let $f^*$ be the (real-valued) function with Fourier coefficients $V[\p^*] = V[-\p^*]= \sqrt{\frac{(2\pi)^{d-1}}{2}}$, and $V[k] = 0$ for $|k| \neq i^*$. Using this, we can lower bound the approximation error as,
\begin{align*}
    \mc E_{\p^*}^\apx \geq \frac{2V[\p^*]^2}{(2\pi)^d}\rbrac{1 - \frac{\abs{G[\p^*]}^2}{\sum_{\m \in \Integer^d} \abs{G[\m N + \p^*]}^2}} 
    &\geq \frac{1}{2\pi} \rbrac{1 - \frac{\abs{G[\p^*]}^2}{\abs{G[\m^* N+\p^*]}^2 + \abs{G[\p^*]}^2}} \\
    &\geq \frac{1}{2\pi(1+C_{3,d})}
\end{align*}

\paragraph{Case 2, $M \leq N$} In this case we show that the noisy estimation error is bounded away from $0$. Define,
$
\displaystyle   \Delta_{\p} := \|G_{\p}\|_1 - \abs{G[\p]}=\sum_{\m \neq 0}\abs{G[\m N+\p]} \geq 0
$.
\Cref{ass:tail_highd} says that for all but $o(N^d)$ terms $\p \in [N]^d$, we have,
\begin{align}
    \Delta_{\p} = \sum_{\m \neq 0}\abs{G[\m N+\p]}\leq C_{1,d} \abs{G[\p]}\sum_{\m \neq0} \round{1 + \|\m\|^2}^{-\frac{d+1}{2}} \leq C_{1,d}\abs{G[\p]} \label{eq:Delta_lower_bound_highd}
\end{align} 
For such an $\p$, we can lower bound the noisy estimation error term $\mc E_\p^\noisy$ as,
\begin{align*}
    \mc E_\p^\noisy&=\frac{\sigma^2}{N^d}\frac{\norm{G_\p}^2}{\norm{G_\p}_1^2}=\frac{\sigma^2}{N^d}\frac{\norm{G_\p}^2}{\rbrac{\abs{G[\p]} + \Delta_{\p}}^2 } 
    \geq 
    \frac{\sigma^2}{N^d}\frac{\abs{G[\p]}^2}{2\abs{G[\p]}^2 + 2\abs{\Delta_\p}^2 }
    \geq \frac{\sigma^2}{2{N^d}(1+C_{1,d})}\\
    \mc E^{\noisy} &= \Omega(\sigma^2),
\end{align*}
since there are $\Omega(N^d)$ such $\p \in[N]^d$ for which equation \cref{eq:Delta_lower_bound_highd} holds.

\begin{proof}[Proof of \Cref{theorem:inconsistent_all_f_highd}]
As in the $d=1$ case, \Cref{theorem:inconsistent_all_f_highd} follows from the proof of \Cref{theorem:inconsistent_all_band_highd}. We showed that for $M\leq N$ (Case 2 above), the noisy estimation error $\mc E^\noisy$ satisfies $\mc E^\noisy=\Omega(\sigma^2)$. Since $\mc E^\noisy$ is independent of the target function $f^*$, the statement of \Cref{theorem:inconsistent_all_f} follows.
\end{proof}

\begin{proof}[Proof of \Cref{thm:inconsistency_monotonic_highd}]
We showed that for $M\leq N$ (Case 2 above), the noisy estimation error $\mc E^\noisy$ satisfies $\mc E^\noisy=\Omega(\sigma^2)$ for all functions. We start by the monotonic boundedness of $G$, we have
\begin{align*}
    \frac{|G[\p]|^2}{\norm{G_{\p}}^2}=\frac{|G[\p]|^2}{\abs{G[\p]}^2+\sum_{k\in(\Integer*)^d}\abs{G[\p+\k N]}^2} \leq \frac{1}{1+\upsilon_d}
\end{align*}
for $\|\p\|_{\infty}<N/2$, where $\upsilon_d=C_d^2(C_d'')^2/c_d^2 > 0$.  

Now consider $f^* = \sum_{\k \in \Integer^d} V[\k] \phi_{\k}$, and suppose $\|f^*\|^2=\sum_{\k\in\Integer^d}\abs{V[\k]}^2=1$. By the smoothness condition on the target function $f^*$, it has Fourier series coefficients that decay as $|V[\k]| \leq B_d (1 + \|\k\|^2)^{-{\frac{d+1}{2}}}$ for all $\k\in\Integer^d$, for a sufficiently large constant $B_d$. Using this we get the inequalities,

\begin{align*}
    \sum_{\norm{\p}>N/2} \abs{V[\p]} &\leq \wt C_d\int_{N/2}^\infty \frac{r^{d-1}}{(1+r^2)^{(d+1)/2}}\dif r =O(\frac1N)\\
    \sum_{\norm{\p}<N/2} \abs{V[\p]}^2 &\geq 1 -  \wh C_d\int_{N/2}^\infty \frac{r^{d-1}}{(1+r^2)^{d+1}}\dif r =1-O(N^{-d-2})
\end{align*}
where we have used Wolfram Alpha to obtain order bounds. The integral involves the special function $_2F_1$, also known as, the hypergeometric function.
The rest of the proof proceeds in a similar manner as the $d=1$ case. The only difference is that the range of $\eps$ is now $\eps\in[0,\frac{d}2].$

\end{proof}

\subsection{Special cases of kernels for \texorpdfstring{$d>1$}{d>1}}
\label{sec:assumptions_highd}

\begin{proof}[Proof of \Cref{lemma:ex_kernels_inconsistent_highd}] \hfill

\subparagraph{Gaussian kernel}~For the wrapped Gaussian kernel, we have
\begin{align*}
    G[\k] = \exp\round{-\|\k\|^2/M^2}
\end{align*}
Therefore, the monotonicity is satisfied with $f(\|\x\|) = \exp(-\|\x\|^2)$. 

\subparagraph{Laplace kernel}
For the wrapped Laplace kernel, we have
\begin{align*}
    G[\k] = \round{1 + \frac{\|\k\|^2}{M^2}}^{-\frac{d+1}{2}}
\end{align*}
(See \cite{rakhlin2019consistency} for a derivation). Therefore, the monotonicity is satisfied with \\$f(\|\x\|) = \round{1 + \|\x\|^2}^{-\frac{d+1}{2}}$. 

\subparagraph{Cauchy kernel}
For the wrapped Cauchy kernel, we have (from \cite{bian1991properties}),
\begin{align*}
    G[\k] = \exp(-\|\k\|/M)
\end{align*}
Therefore, the monotonicity is satisfied with $f(\|\x\|) = \exp(-\|\x\|)$. 
\end{proof}

\begin{lemma}
\label{lemma:coeff_inspan_highd}
For $\betavec=(\beta_\p)\in\Complex^{N^d}$, the Fourier series of the function $\inner{\betavec,K(X_n,\cdot)}_n$ is, 
\begin{align*}
    B[\k] = N^{d/2} \betavec\tran{\u}_{\k \mod N}\cdot G[\k],\quad \k\in\Integer^d
\end{align*}
\end{lemma}
The proof for this lemma is provided in \Cref{app:additional_proofs} in the supplementary materials.

\section{Miscellaneous results}
\label{sec:assumptions}

The proof of the following for $d>1$ is provided in \Cref{sec:assumptions_highd}. 

\begin{proof}[Proof of Proposition \ref{lemma:ex_kernels_inconsistent} (d=1)]

We state what it means for a kernel to be wrapped on the unit circle \cite{mardia2000directional}. This means that the wrapped kernel evaluation at $t \in \unitcircle$ has contributions from the Euclidean kernel at $t + 2\pi k$ for all $k \in \Integer$. In particular,
\begin{definition}[Wrapper kernel]
\label{def:wrapped}
A kernel $\wt{g} : \Real \rightarrow \Real^+$ can be wrapped to define a wrapped kernel $g : \unitcircle \rightarrow \Real^+$ given by:
\begin{align*}
    g(t) = \sum_{k \in \Integer} \wt{g}(t + 2\pi k)
\end{align*}
\end{definition}
It is a fact that the Fourier series coefficients of the wrapped on the unit circle are equal to the Fourier transform at integer values on the real line, \ie,
\begin{align*}
    G[k] = \int \wt g(t)\exp(-jkt)\dif t\qquad \forall\ k\in\Integer.
\end{align*}
Using this, we derive the proposition for the given kernels \cite{mardia2000directional}.

\subparagraph{Gaussian kernel} The wrapped Gaussian kernel satisfies 
$
    G[k] = e^{-k^2/4M^2}.
$
Therefore, $f(x) = e^{-x^2}$ clearly satisfies the monotonicity condition.

\subparagraph{Laplace kernel} For the wrapped Laplace kernel, $\displaystyle G[k] = \frac{1}{M^2 + k^2}$. Thus, the Laplace kernel satisfies \Cref{con:monotonic} with $f(x) = \frac{1}{1 + x^2}$.

\subparagraph{Cauchy kernel} For the Cauchy kernel, $\displaystyle G[k] = \exp(-|k|/M)$. Therefore, \Cref{con:monotonic} is satisfied with $f(x) = \exp(-|x|)$.

\end{proof}

\begin{proof}[Proof of \Cref{prop:monotone_assumptions}]
We first show boundedness. As $f(0)< \infty$, we have $f(k) < \frac{f(0)}{1 + k^2}$ for all $k \in \Real^+$. Therefore, $\sum_{k \in \Integer} |G[k]| \leq C \int_k f(k) < \infty$. Then, $K(x,x') \leq \sum_{k} |G[k]| < \infty$.

For the tail assumption, we have for $M' > M$,
\begin{align*}
    \frac{G[M'k + i]}{G[i]} \leq \frac{C'' C}{c} \frac{f((i+k)/M)}{f(i/M)} \frac{C''C}{c} \frac{1}{1 + k^2}~.
\end{align*}

For the head assumption, we have for $M' \leq M$,
\begin{align*}
    0 < \frac{G[i]}{G[i+M']} \leq \frac{c}{C''C} \frac{f(i/M)}{f((i+M')/M)} \leq \frac{c}{C''C C'}~.
\end{align*}
These prove the proposition.
\end{proof}

\subsection{Intermediate Lemmas}

\begin{lemma}
\label{lemma:noise_dft}
Let $\xivec$ be a random vector with $\Exp\xivec\xivec\herm=\sigma^2\I,$ then for $\u\in\Complex^N$,
    $\Exp_{\xivec} 
    |\xivec\herm\Kmat\inv\u |^2 = 
    \sigma^2 \cdot{\u\herm\Kmat^{-2} \u}$
\end{lemma}
\begin{proof}
$    \Exp_{\xivec} 
    |\xivec\herm\Kmat\inv\u |^2 = 
    \Exp_{\xivec} \u\herm\Kmat\inv\xivec\xivec\herm\Kmat\inv\u = \sigma^2\u\herm\Kmat^{-2}\u.
$
\end{proof}

\begin{lemma}
\label{lemma:coeff_inspan}
For $\betavec\in\Complex^N$, let $B$ be the Fourier series of the function $\inner{\betavec,K(X_N,\cdot)}_N$. Then,
\begin{align*}
    B[k] = \sqrt N\betavec\tran{\u}_{k \mod N}\cdot G[k],\quad k\in\Integer
\end{align*}
\end{lemma}
\begin{proof} Use the Fourier series definition to get,
\begin{align*}
    &\sum_{i=1}^{N}\beta_i \frac1{2\pi}\int_{-\pi}^\pi K(x_i,t)\dif t=\sum_{i=1}^{N}\beta_i \frac1{2\pi}\int_{-\pi}^\pi g(M(x_i-t) \mod \unitcircle)\dif t\\
    &=\sum_{i=1}^{N}\beta_i\frac1{2\pi}\int_{-\pi}^\pi g(M\tau)e^{-jk\tau}e^{-jkx_i}\dif \tau = \sqrt{N}\beta\tran\u_{k\mod N} G[k]
\end{align*}
since $e^{-jkx_i}=e^{-j\frac{2\pi}N ki}=e^{-j\frac{2\pi}N (k\mod N)i}=\sqrt{N}\u_{k\mod N}$.
This concludes the proof.
\end{proof}

\paragraph{Eigenfunctions of \texorpdfstring{$\mc T_K$,  $\mc T_K^N$}{TK} and eigenvectors of \texorpdfstring{$\Kmat$}{Kmat}}
\label{sec:technical:eigenfunctions}

\begin{proof}[Proof of \Cref{prop:eigenvectors_of_Kmat}]
It suffices to show the eigenvector equation for the unnormalized version of $\u_\ell.$ We start by noting that
\begin{align*}
    \Kmat_{im'}=g(M(x_{m'}-x_i))=\sum_{m\in\Integer}G[m]e^{jm(x_{m'}-x_i)}=\sum_{m\in\Integer}G[m]e^{j\frac{2\pi}{N}m(m'-i)}.
\end{align*}
Using this, we have
\begin{subequations}
\begin{align*}
    &(\Kmat\u_{\ell})_i = \sum_{m^\prime=0}^{N-1} \Kmat_{im^\prime} e^{-j\frac{2\pi}{N}m'\ell } = 
    \sum_{m'} \sum_{m \in \Integer} G[m] e^{j \frac{2\pi }{N}m ( m^\prime - i)} e^{-j\frac{ 2\pi  }{N}m^\prime \ell} \\
    &= \sum_{m \in \Integer} G[m] e^{-j \frac{2\pi }{N}m i} \sum_{m^\prime=0}^{N-1}  e^{j \frac{2\pi}{N} (m-\ell) m^\prime} 
    =  N \sum_{m \in \Integer} G[mN + \ell] e^{-j \frac{2\pi }{N}(mN + \ell) i} \\
    &=   Ne^{-j \frac{2\pi }{N}\ell i}  \sum_{m \in \Integer} G[mN + \ell] e^{-j 2\pi m i} 
    =   e^{-j \frac{2\pi }{N}\ell i} N \sum_{m \in \Integer} G[mN + \ell] = e^{-j \frac{2\pi }{N}\ell i} \cdot N \norm{G_\ell}_1
\end{align*}
\end{subequations}
This proves $\Kmat\u_\ell=N\norm{G_\ell}_1\u_\ell$. The rest follows from standard results on linear algebra.
\end{proof}

\begin{proof}[Proof of \texorpdfstring{\Cref{prop:eigenfunctions_of_TK}}{}]
Observe that
\begin{align*}
    &\mc T_{K}\curly{\phi_k}(x) = \frac1{2\pi}\int_{-\pi}^\pi K(x,x^\prime)\phi_k(x^\prime)\dif x^\prime = \frac1{2\pi}\int_{-\pi}^\pi g(M(x^\prime-x \mod \unitcircle))e^{jkx^\prime}\dif x^\prime \\
    &= e^{jkx}\cdot\frac1{2\pi}\int_{-\pi}^{\pi} g(Mu)e^{-jku}\dif u = G[k]\phi(x)
\end{align*}
This proves the claim.\end{proof}

The following lemma relates the eigenfunctions of the empirical covariance operator defined in equation \cref{eq:def:emp_cov} to the eigenvectors of the kernel matrix.
\begin{lemma}[Eigenfunctions of \texorpdfstring{$\Kcov$}{TKn}]\label{lemma:eigvectors_covariance}
Let $(\lambda, \psi)$ be an eigenvalue-eigenfunction pair of $\Kcov$. Assume $\Kmat$ is invertible. Then for $\lambda>0$, a unit-norm eigenfunction $\psi$ satisfies,
\begin{align}
    \label{eq:psi_to_e}
    \psi = \sum_{i=1}^n\frac{e_i}{\sqrt{n\lambda}},K(x_i,\cdot),
\end{align}
where $\e=(e_i)\in\Complex^n$ is a unit-norm eigenvector of $\Kmat$ satisfying,
$
    \Kmat\e  = n\lambda\e.
$
\end{lemma}

We apply the above lemma to the setting described in \Cref{sec:model}. The proof is provided in \Cref{app:additional_proofs} in the supplementary materials.
\begin{lemma}[Eigenfunctions of $\mc T_K^N$]\label{lem:eigenfunctions_of_Kcov}
The eigenfunctions for $\mc T_{{K}}^N$ are,
\begin{align*}
    \psi_\ell &= \frac{1}{\sqrt{\norm{G_\ell}_1}} \sum_{m \in \Integer} G[mN + \ell] \phi_{mN + \ell},\qquad \ell\in[N].
\end{align*}
They satisfy,
$
    \mc T_{{K}}^N{\psi_\ell}=\norm{G_\ell}_1\psi_\ell,$ and their norms satisfy $
    \norm{\psi_\ell}_\Hilbert = 1,\text{ and }
    \norm{\psi_{\ell}} = \frac{1}{\sqrt{\norm{G_\ell}_1}} \norm{G_\ell}.
$
Furthermore, $\psi_\ell$ are orthogonal in $L^2$, \ie, $\inner{\psi_\ell, \psi_{k}} = 0$ for $k \neq \ell$. \end{lemma}
\begin{proof}[Proof of \Cref{lem:eigenfunctions_of_Kcov}]
By \Cref{lemma:eigvectors_covariance}, we have 
\begin{align*}
    \psi_{\ell} &= \inner{\frac{\u_{\ell}}{\sqrt{N \norm{G_\ell}_1}}, K(X_N,\cdot)}_N  = \inner{\frac{\wb{\u_{\ell}}}{\sqrt{N \norm{G_\ell}_1}}, K(X_N,\cdot)}_N 
\end{align*}
Then using \Cref{lemma:coeff_inspan} we have a Fourier series expansion of the form
\begin{align*}
    \psi_{\ell} &=  \sum_{k \in \Integer} \sqrt{N}\frac{\u_\ell\herm}{\sqrt{N\norm{G_\ell}_1}} \u_{k\mod N} G[k] \phi_{k} = \frac{1}{\sqrt{\norm{G_\ell}_1}}\sum_{k \in \Integer} \u_\ell\herm\u_{k\mod N} G[k] \phi_{k}\\
    &=\frac{1}{\sqrt{\norm{G_\ell}_1}}\sum_{m\in\Integer}G[mN+\ell]\phi_{mN+\ell}
\end{align*}

To see the orthogonality, suppose $k,\ell\in\curly{0,1,\ldots,N-1}$ and $k\neq \ell.$ Then 
\begin{align*}
    \inner{\psi_{\ell},\psi_k} &= \frac{1}{\sqrt{\norm{G_\ell}_1\norm{G_k}_1}} \sum_{m,m^\prime \in \Integer} G[mN + \ell] G[m^\prime N + k] \inner{\phi_{mN+\ell},\phi_{m'N+k}} 
    = 0 
\end{align*}
For $L^2$ norm, substitute $k=\ell$ above to get,
\begin{align*}
\inner{\psi_{\ell}, \psi_{\ell}} &= \frac{1}{\norm{G_\ell}_1} \sum_{m, m' \in \Integer} {G[mN + \ell]G[m'N + \ell]} \inner{\phi_{mN+\ell},\phi_{m'N+\ell}} 
=\frac{1}{\norm{G_\ell}_1}  \norm{G_\ell}^2
\end{align*}
This proves the claim.
\end{proof}

\section{Proofs to technical lemmas}

\begin{proposition}[Parseval's theorem in high dimensions]
\label{thm:highd_parseval_continuous}
For a function $f:\unitcircle^d\rightarrow\Real$ with Fourier series coefficients $F[\k]$ for $k \in \Integer^d$, we have
\begin{align*}
    \sum_{\k \in \Integer^d} |F[{\k}]|^2 = \frac{1}{(2\pi)^d} \int_{\unitcircle^d} |f(\t)|^2 \dif \t.
\end{align*}
\end{proposition}

\paragraph{Eigenfunctions of \texorpdfstring{$\mc T_K$, $\Kcov$}{TK} and eigenvectors of \texorpdfstring{$\Kmat$}{Kmat} for \texorpdfstring{$d>1$}{d>1}}
\label{sec:technical:eigenfunctions_highd}
The proofs of the following two statements are provided in \Cref{app:additional_proofs}.

\begin{proposition}
\label{prop:eigenvectors_of_Kmat_highd}
\Cref{prop:eigenvectors_of_Kmat} holds with $\u_\l$ from \Cref{def:dft_d>1} and $\Kmat=(K(\x_\p,\x_{\p'}))\in\Real^{N^d\times N^d},$ with eigenvalue $\lambda_\l=N^d \norm{G_\l}_1$, \ie, $\Kmat\u_\l=\lambda_\l\u_\l$.
\end{proposition}

\begin{lemma}[Eigenfunctions of $\mc T_K^{N,d}$]\label{lem:eigenfunctions_of_Kcov_highd}
The eigenfunctions for the empirical operator $\mc T_{{K}}^{N,d}$ are, 
\begin{align*}
    \psi_\l &= \frac{1}{\sqrt{\norm{G_\l}_1}} \sum_{\m \in \Integer^d} G[\m N + \l] \phi_{\m N + \l},\qquad \l\in[N]^d.
\end{align*}
They satisfy,
$
    \mc T_{{K}}^{N,d}\curly{\psi_\l}=\norm{G_\l}_1\psi_\l,$ and their norms satisfy $
    \norm{\psi_\l}_\Hilbert = 1,$ as well as $    \norm{\psi_{\l}} = \frac{1}{\sqrt{\norm{G_\l}_1}} \norm{G_\l}.
$
Furthermore, $\psi_\l$ are orthogonal in $L^2$, \ie, $\inner{\psi_\l, \psi_{\k}} = 0$ for $\k \neq \l$. 
\end{lemma}

\subsection{Approximation error: Proof of \texorpdfstring{\Cref{lem:MSE_highd}\ref{lem:apx_highd}}{}}
\label{sec:approx_proof_highd}

Once again, the proof proceeds by applying the Pythagorean theorem to the triangle $\curly{0,f^*,P_{\X}f^*}$ in $L^2$. The following lemma gives exact expressions for projection of the target function and its norm.

\begin{lemma}[Projection] For $f^*=\sum_{\k\in\Integer^d}V[\k]\phi_\k$
\label{lemma:projection_highd}
\begin{align*}
    P_{\X}f^* = \sum_{\l \in [N]^d} \frac{\inner{G_\l,V_\l}}{{\norm{G_\l}^2}} \sum_{\m \in \Integer^d} G[\m N + \l] \phi_{\m N + \l},\qquad\text{and}\qquad
    \norm{P_{\X}f^*}^2 = \sum_{\l \in [N]^d} \frac{\inner{G_\l,V_\l}^2}{\norm{G_\l}^2}
\end{align*}
\end{lemma}

We get,
\begin{align*}
    \norm{f^* - P_{\X}f^*}^2 
    = \norm{f^*}^2 - \norm{P_{\X}f^*}^2 
    = \norm{V}^2 - \sum_{\l \in [N]^d} \frac{\inner{G_\l,V_\l}^2}{\norm{G_\l}^2}
\end{align*}

\begin{proof}[Proof of \Cref{lemma:projection_highd}]
Note that \Cref{lem:eigenfunctions_of_Kcov_highd} shows that $\curly{\frac{\psi_\l}{\|\psi_\l\|}}_{\l \in [N]^d}$ is an orthonormal basis for $\Span\{K(\x_{\l},\cdot)\}$. Consequently, we have \begin{align*}
    P_{\X}f^* = \sum_{\l \in [N]^d} \inner{f^*, \frac{\psi_\l}{\|\psi_\l\|}} \frac{\psi_\l}{\|\psi_\l\|},\qquad\text{and}\qquad
    \norm{P_{\X}f^*}^2 = \sum_{\l \in [N]^d} \inner{f^*, \frac{\psi_\l}{\|\psi_\l\|}}^2
\end{align*} 
We compute these projections. For $\l \in [N]^d$,
\begin{align*}
    \inner{f^*, \psi_\l} &= \frac{1}{\sqrt{\norm{G_\l}_1}} \inner{ \sum_{\k \in \Integer^d} V[\k] \phi_\k, \sum_{\m \in \Integer^d} G[\m N + \l] \phi_{\m N + \l} } \\
    &= \frac1{\sqrt{\norm{G_\l}_1}} \sum_{\m,\k \in \Integer^d} G[\m N + \l] V[\k] \inner{\phi_{\k}, \phi_{\m N + \l} }\\
    &= \frac1{\sqrt{\norm{G_\l}_1}} \sum_{\m \in \Integer^d} G[\m N + \l] V[\m N + \l]=\frac{\inner{G_\l,V_\l}}{\sqrt{\norm{G_\l}_1}}
\end{align*}
Thus, we get that,
\begin{align*}
    \inner{f^*, \frac{\psi_\l}{\|\psi_\l\|}} \frac{\psi_\l}{\|\psi_\l\|} = \frac{\inner{G_\l,V_\l}}{\norm{G_\l}^2}\sum_{\m\in\Integer^d}G[\m N+\l]\phi_{\m N+\l}
\end{align*}
The claims follow immediately.
\end{proof}

\subsection{Noise-free estimation error: Proof of \texorpdfstring{\Cref{lem:MSE_highd}\ref{lem:nfe_highd}}{}}
\label{sec:noisefree_proof_highd}

Let $F$ be the fourier series of $\inner{\Kmat\inv R_n\curly{f^*-P_{\X}f^*},K(X_n,\cdot)}_n$. From \Cref{lemma:coeff_inspan_highd} we have,
\begin{align*}
    F[\k] = \sqrt{N^d}R_n \cbrac{f^* - P_{\X}f^*}\tran\Kmat\inv\u_{\k \mod N^d} \cdot G[\k]
\end{align*}
By Parseval's theorem (\Cref{thm:highd_parseval_continuous}), we conclude,
\begin{align*}
    \frac{1}{(2\pi)^d} \int_{\unitcircle} \rbrac{\inner{\K\inv R_n \cbrac{f^* - P_{\X}f^*},K(X_n,t)}_n}^2\dif t = \sum_{\k \in \Integer^d} |F[\k]|^2
\end{align*}
We will show that
\begin{align*}
    {R_n\cbrac{f^* - P_{\X}f^*}\tran \K\inv\u_\l}
    &= \rbrac{\frac{\inner{V_{\l},\one}}{\inner{G_{\l},\one}} - \frac{\inner{G_\l , V_\l}}{\|G_\l\|^2}}
\end{align*}

We have $\Kmat\inv\u_{\l} = \u_{\l}\cdot \frac{1}{N^d \norm{G_{\l}}_1}$.
By \Cref{lemma:projection}, we can write $P_{\X}f^*$ on the data as
\begin{align*}
    P_{\X}f^*(\x_\p) &= \sum_{\l \in [N]^d} \frac{\inner{G_\l,V_\l}}{\norm{G_\l}^2} \sum_{\m \in\Integer^d}G[\m N+\l] \phi_{\m N+\l}(\x_\p)= \sum_{\l\in [N]^d} \frac{\inner{G_\l,V_\l}}{\norm{G_\l}^2} \norm{G_\l}_1 \wb{u_{\l \p}}\\
    f^*(\x_\p)&=\sum_{\k\in\Integer^d}V[\k]\phi_\k(\x_\p)=\sum_{\l \in [N]^d} \inner{V_{\l},\one}\wb{u_{\l \p}}
\end{align*}
We thus have
\begin{align*}
    R_n \cbrac{f^* - P_{\X}f^*}\tran\Kmat\inv\u_{\l} &=  \sum_{\p,\l' \in [N]^d} \round{\inner{V_{\l'},\one}-\frac{\inner{G_{\l'},V_{\l'}}}{\norm{G_{\l'}}^2}\norm{G_{\l'}}_1} \frac{\wb{u_{\l' i}}u_{\l i}}{N^d\norm{G_\l}_1}\\
    &=
    \frac{1}{N^d}\round{\frac{\inner{V_{\l},\one}}{\norm{G_\l}_1}-\frac{\inner{V_\l,G_\l}}{\norm{G_\l}^2}}
\end{align*}
This gives,
\begin{align*}
    \sum_{\k\in\Integer^d}\abs{F[\k]}^2=\frac{1}{N^d}\sum_{\l \in [N]^d} \round{\frac{\inner{V_{\l},\one}}{\inner{G_{\l},\one}}-\frac{\inner{V_\l,G_\l}}{\norm{G_\l}^2}}^2\norm{G_\l}^2.
\end{align*}

\subsection{Noisy estimation error: Proof of \texorpdfstring{\Cref{lem:MSE_highd}\ref{lem:nye_highd}}{}}
\label{sec:noisy_proof_highd}

Similar to $d=1$, we derive this by an application of Parseval's theorem.

Define the Fourier series,
\begin{align*}
    \inner{\Kmat\inv\xivec,K(X_n,\t)}_n = \sum_{\k\in\Integer^d} E[\k] \exp\round{j \inner{\k,\t}}
\end{align*}
By \Cref{thm:highd_parseval_continuous} (Parseval's theorem), we have, 
\begin{subequations}
\begin{align*}
    &\Exp_{\xivec} \sbrac{ \frac{1}{(2\pi)^d} \int_{\unitcircle} |\inner{\Kmat\inv\xivec,K(X_n,\t)}_n|^2 \dif \t } = \sum_{\k \in\Integer^d} \Exp_{\xivec} |E[\k]|^2 
    = \sum_{\p \in [N]^d} \sum_{\m \in \Integer^d} \Exp_{\xivec} \abs{E[\m N + \p]}^2 \\
    &= \sum_{\p \in [N]^d} \sum_{\m \in \Integer^d} \left|G[\m N+\p]\right|^2 \Exp_{\xivec} \left|\xivec\tran\Kmat\inv\u_i \right|^2\cdot N^d = \sigma^2 \sum_{\p \in [N]^d} \norm{G_\p}^2  \left( \u_\p \herm\Kmat^{-2} {\u}_\p  \right)\cdot N^d \\
    &= {\sigma^2} \sum_{\p \in [N]^d} \norm{G_\p}^2\frac{1}{N^{2d} \norm{G_\p}_1^2} N^d = {\sigma^2} \sum_{\p \in [N]} \frac{\norm{G_\p}^2}{N^d \norm{G_\p}_1^2}
\end{align*}
\end{subequations}
where we have used \Cref{lemma:coeff_inspan_highd} in the second, and \Cref{lemma:noise_dft} in the third, and \Cref{prop:eigenvectors_of_Kmat_highd} in the last line.

\subsection{Additional Proofs}
\label{app:additional_proofs}

\begin{proof}[Proof of \Cref{lemma:eigvectors_covariance}]
We will first show that $\psi$ can be written as a linear combination of the $n$ representers $\set{K(x_i,\cdot)}$.
\begin{align}
    \label{eq:eigfunc_Kcov_representers}
    \psi=\sum_{i=0}^{n-1} \beta_i K(x_i,\cdot)
\end{align}
Let $\psi \in \mc H$ be an eigenfunction of $\Kcov$ with eigenvalue $\lambda$. Then by definition of $\Kcov$ we have,
\begin{equation}\label{eigenvec1}
    \lambda\psi=\Kcov({\psi}) = \frac{1}{n} \sum_{i=1}^{n} \langle K(x_i,\cdot), \psi \rangle_{\Hilbert} K(x_i,\cdot)
    = \frac{1}{n} \sum_{i=1}^{n} {\psi(x_i)} K(x_i,\cdot)
\end{equation}
where the last equality holds due to the reproducing property of the kernel. Define $\beta_i = \frac{\psi(x_i)}{n\lambda}$ to show \cref{eq:eigfunc_Kcov_representers}.
Next, rewriting the equation for an eigenfunction $\psi$, expressed as \cref{eq:eigfunc_Kcov_representers}, we get
\begin{equation}
\Kcov \left(\sum_{i=1}^{n} \beta_i K(x_i,\cdot) \right)  = \lambda \sum_{i=1}^{n} \beta_i K(x_i,\cdot).
\end{equation}
By definition of $\Kcov$ however we get,
\begin{align}\label{eq:eval_Kcov_span}
\Kcov \left(\sum_{i=0}^{n} \beta_i K(x_i,\cdot) \right) &= \frac{1}{n} \sum_{i,j=1}^{n} \beta_i \langle K(x_j,\cdot), K(x_i,\cdot) \rangle_{\Hilbert} K(x_j,\cdot) = \frac1n\sum_{j=1}^{n} (\Kmat\betavec)_j K(x_j,\cdot)
\end{align}

Evaluating functions on the RHS of equations \eqref{eq:eigfunc_Kcov_representers} and \eqref{eq:eval_Kcov_span} at $x_\ell$ yields,
\begin{equation*}
\frac{1}{n} \cdot \sum_{i=1}^{n} \sum_{j=1}^{n} \beta_i K(x_i,x_j) K(x_j,x_l) =
\lambda \sum_{i=1}^{n} \beta_i K(x_i,x_l)\qquad\text{for\ all\ } \ell\in\{0,1,\ldots,n-1\}
\end{equation*} 
Compactly these $n$ equations can be written as:
\begin{equation*}
\Kmat^2 \betavec = n\lambda \Kmat \betavec \implies \Kmat \betavec = n\lambda\betavec
\end{equation*} 
since $\Kmat$ is inverible. Thus $\betavec$ is a scaled eigenvector of $\Kmat$. It remains to determine the scale of $\betavec$ that defines $\psi$.

Now, the norm of $\psi$ can be simplified as
\begin{align*}
 \norm{\psi}_\Hilbert^2 
 = \inner{\sum_{i=1}^{n} \beta_i K(x_i,\cdot),\sum_{j=1}^{n} \beta_j K(x_j,\cdot)}_\Hilbert &= \sum_{i,j=1}^{n} \beta_i\wb{\beta_j}\inner{K(x_i,\cdot),K(x_j,\cdot)}_\Hilbert \\
 &=
 \betavec\herm\Kmat\betavec=n\lambda\norm{\betavec}^2.
\end{align*}
Since $\psi$ is unit norm, we have $\norm{\betavec}=\frac{1}{\sqrt{n\lambda}}$. This concludes the proof.
\end{proof}

\begin{proof}[Proof of \Cref{lemma:coeff_inspan_highd}] Use the Fourier series definition to get,
\begin{align*}
    &\sum_{\p \in [N]^d} \beta_\p \frac1{(2\pi)^d}\int_{\unitcircle} g(M(\t-\x_\p \mod \unitcircle))\exp\round{-j\inner{\k,\t}}\dif \t \\
    &\quad= \sum_{\p \in [N]^d}  \frac{\beta_\p}{(2\pi)^d} \int_{\unitcircle} g(M\bm{\tau})\exp\round{-j\inner{\k,\bm{\tau}}}\exp\round{-j\inner{\k,\x_\p}}\dif \bm{\tau} 
    = N^{\nicefrac{d}{2}}\beta\tran\u_{\k \mod N} G[\k]
\end{align*}
This concludes the proof.
\end{proof}

\begin{proof}[Proof of \Cref{prop:eigenvectors_of_Kmat_highd}]
It suffices to show the eigenvector equation for the unnormalized version of $\u_\l.$ We start by noting that
\begin{align*}
    \Kmat_{\p \m'}=g(M(\x_{\m'}-\x_\p))=\sum_{\m\in\Integer^d}G[\m]e^{j \inner{\m,\x_{\m'}-\x_\p})}
    =\sum_{\m\in\Integer^d}G[\m]e^{j\frac{2\pi}{N} \inner{\m,\m'-\p}}.
\end{align*}
Using this, we have
\begin{subequations}
\begin{align*}
    &(\Kmat \u_{\l})_{\q} = \sum_{\p \in [N]^d} \Kmat_{\q,\p} u_{\l \p} = \sum_{\p} \sum_{\m' \in \Integer^d} G[\m'] \exp\round{j \frac{2\pi }{N}\inner{\m', \p-\q}} \exp\round{\frac{-j 2\pi  }{N}\inner{\p, \l}} \\
    &= \sum_{\m' \in \Integer^d} G[\m'] e^{-j \frac{2\pi }{N}\inner{\m',\q}} \sum_{\p \in [N]^d}  e^{j \frac{2\pi}{N}\inner{(\m'-\l), \p}} 
    = N^d \sum_{\m \in \Integer^d} G[\m N + \l] e^{-j \frac{2\pi}{N}\inner{\m N + \l, \q}} \\
    &= e^{-j \frac{2\pi}{N} \inner{\l, \q}} N^d \sum_{\m \in \Integer^d} G[\m N + \l] 
    = e^{-j \frac{2\pi }{N} \inner{\l, \q}} N^d \lambda_\l
\end{align*}
\end{subequations}
This proves $\Kmat\u_\l=N^d\norm{G_\l}_1\u_\l$. The rest follows from standard results on linear algebra.
\end{proof}

\begin{proof}[Proof of \Cref{lem:eigenfunctions_of_Kcov_highd}]
By \Cref{lemma:eigvectors_covariance}, we have 
\begin{align*}
    \psi_{\l} &= \inner{\frac{\u_{\l}}{\sqrt{N^d \norm{G_\l}_1}}, K(X_n,\cdot)}_n  = \inner{\frac{\wb{\u_{\l}}}{\sqrt{N^d \norm{G_\l}_1}}, K(X_n,\cdot)}_n
\end{align*}
Then using \Cref{lemma:coeff_inspan_highd} we have a Fourier series expansion of the form
\begin{align*}
    \psi_{\l} &=  \sum_{\k \in \Integer^d} \sqrt{N^d}\frac{\u_\l\herm}{\sqrt{N^d\norm{G_\l}_1}} \u_{\k \mod N} G[\k] \phi_{\k} = \frac{1}{\sqrt{\norm{G_\l}_1}}\sum_{\k \in \Integer^d} \u_\l\herm\u_{\k \mod N} G[\k] \phi_{\k}\\
    &=\frac{1}{\sqrt{\norm{G_\l}_1}}\sum_{\m\in\Integer^d}G[\m N+\l]\phi_{\m N+\l}
\end{align*}

To see the orthogonality, suppose $\k,\l\in [N]^d$ and $\k \neq \l.$ Then 
\begin{align*}
    \inner{\psi_{\l},\psi_\k} &= \frac{1}{\sqrt{\norm{G_\l}_1\norm{G_\k}_1}} \sum_{\m,\m' \in \Integer^d} G[\m N + \l] G[\m' N + \k] \inner{\phi_{\m N+\l},\phi_{\m'N+\k}} 
    = 0 
\end{align*}
For $L^2$ norm, substitute $\k=\l$ above to get,
\begin{align*}
\inner{\psi_{\l}, \psi_{\l}} &= \frac{1}{\norm{G_\l}_1} \sum_{\m, \m' \in \Integer^d} {G[\m N + \l]G[\m'N + \l]} \inner{\phi_{\m N+\l},\phi_{\m'N+\l}} 
=\frac{1}{\norm{G_\l}_1}  \norm{G_\l}^2
\end{align*}
This proves the claim.
\end{proof}

\end{document}


Index of Supplementary Materials

Title of paper: On the Inconsistency of Kernel Ridgeless Regression in Fixed Dimensions

Authors: Daniel Beaglehole, Mikhail Belkin, Parthe Pandit

File: 	Inconsistency\_of\_Kernel\_Interpolation\_supplementary\_materials.pdf	

Type: Mathematical proofs.

Contents: The supplementary materials contain omitted technical lemmas and proofs, as well as extensions of our results to more than one dimension. All results (those for $d>1$) are included in a Section D. The first unlabeled subsection contains the setup for $d>1$. Sections D.1-D.3 provide statements and proofs of the main results and the MSE decomposition in $d>1$. Section D.4 contains proofs that special kernels satisfy the monotonic boundedness condition. Sections E and F contain miscellaneous technical results and other ommitted proofs.